\documentclass{article}

     \usepackage[nonatbib, preprint]{neurips_2019}

\usepackage[utf8]{inputenc} % allow utf-8 input
\usepackage[T1]{fontenc}    % use 8-bit T1 fonts
\usepackage{hyperref}       % hyperlinks
\usepackage{url}            % simple URL typesetting
\usepackage{booktabs}       % professional-quality tables
\usepackage{amsfonts}       % blackboard math symbols
\usepackage{nicefrac}       % compact symbols for 1/2, etc.
\usepackage{microtype}      % microtypography

\usepackage[dvipsnames]{xcolor}
\RequirePackage{fancyhdr}
\usepackage{times}
\usepackage{amsmath,amsfonts,bm}

\def\figref#1{figure~\ref{#1}}
\def\eqref#1{equation~\ref{#1}}
\def\1{\bm{1}}

\def\ry{{\textnormal{y}}}

\def\rvx{{\mathbf{x}}}

\def\vu{{\bm{u}}}
\def\vv{{\bm{v}}}
\def\vw{{\bm{w}}}
\def\vx{{\bm{x}}}
\def\vy{{\bm{y}}}

\def\mA{{\bm{A}}}
\def\mB{{\bm{B}}}

\def\mH{{\bm{H}}}
\def\mI{{\bm{I}}}

\def\mT{{\bm{T}}}
\def\mU{{\bm{U}}}
\def\mV{{\bm{V}}}
\def\mW{{\bm{W}}}
\def\mX{{\bm{X}}}
\def\mY{{\bm{Y}}}

\DeclareMathAlphabet{\mathsfit}{\encodingdefault}{\sfdefault}{m}{sl}
\SetMathAlphabet{\mathsfit}{bold}{\encodingdefault}{\sfdefault}{bx}{n}

\DeclareMathOperator*{\argmin}{arg\,min}

\DeclareMathOperator{\sign}{sign}

\usepackage{hyperref}
\usepackage{url}
\usepackage{amsmath}
\usepackage{algorithm}
\usepackage{algorithmic}
\usepackage[utf8]{inputenc}
\usepackage{pgfplots}
\usepackage{siunitx}
\usepackage{pdfpages}

\usepgfplotslibrary{groupplots}
\usetikzlibrary{arrows.meta}
\usepackage{wrapfig}

\newcommand{\naturals}{\mathbb{N}}
\newcommand{\reals}{\mathbb{R}}
\newcommand{\inner}[1]{\langle #1 \rangle}
\newcommand{\norm}[1]{\left\lVert #1 \right\rVert}
\newcommand{\mean}[2]{\mathbb{E}_{#1} \left[ #2 \right]}
\newcommand{\prob}[2]{\mathbb{P}_{#1}\left[#2\right]}
\newcommand{\BlackBox}{\rule{1.5ex}{1.5ex}}  % end of proof
\newenvironment{proof}{\par\noindent{\bf Proof\ }}{\hfill\BlackBox\\}

\newcommand{\todo}[1]{\textbf{TODO: #1}}
\newcommand{\rank}{\textnormal{rank}}

\def\dD{\mathcal{D}}
\def\nN{\mathcal{N}}

\def\0{\mathbf{0}}

\newtheorem{theorem}{Theorem}

\newtheorem{lemma}{Lemma}
\newtheorem{corollary}{Corollary}
\newtheorem{remark}{Remark}

\title{Decoupling Gating from Linearity}

\author{%
  Jonathan Fiat \\
  School of Computer Science \\
  The Hebrew University, Israel \\
   \And
  Eran Malach \\
  School of Computer Science \\
  The Hebrew University, Israel \\
   \And
  Shai Shalev-Shwartz \\
  School of Computer Science \\
  The Hebrew University, Israel \\
}

\begin{document}

\maketitle

\begin{abstract}

ReLU neural-networks have been in the focus of many recent theoretical works,
trying to explain their empirical success. Nonetheless, there is still a gap between
current theoretical results and empirical observations, even in the case
of shallow (one hidden-layer) networks.
For example, in the task of memorizing a random sample of size $m$ and dimension $d$,
the best theoretical result requires the size of the network to be
$\tilde{\Omega}(\frac{m^2}{d})$
\footnote{We use $\tilde{\Omega}$ to hide constant and logarithmic factors.}, while empirically a network of size slightly larger
than $\frac{m}{d}$ is sufficient. To bridge this gap, we turn to study a simplified
model for ReLU networks. 
We observe that a ReLU neuron is a product of a linear function with a gate 
(the latter determines whether the neuron is active or not), where both share a jointly trained weight vector.
In this spirit, we introduce the Gated Linear Unit (GaLU), which simply decouples the
linearity from the gating by assigning different vectors for each role.
We show that GaLU networks allow us to get optimization and generalization
results that are much stronger than those available for ReLU networks.
Specifically, we show a memorization result for networks
of size $\tilde{\Omega}(\frac{m}{d})$,
and improved generalization bounds. Finally, we show that in some scenarios,
GaLU networks behave similarly to ReLU networks,
hence proving to be a good choice of a simplified model.

\end{abstract}

\section{Introduction}
ReLU neural-networks attracted vast interest in recent years due to their empirical success.
This interest has sparked many theoretical works aiming to explain the behavior
of learning ReLU networks with gradient-based algorithms.
While the theoretical research greatly advanced in the last few years,
there are still many open questions and gaps between our theoretical understanding
and empirical observations. Even in the case of shallow (one hidden-layer) ReLU networks,
current theoretical results do not seem to apply in practice.
Take for example the simple task of memorizing a random sample of $m$ examples
sampled from a $d$-dimensional Gaussian distribution. 
As far as we know,
the best result in the literature shows that a ReLU neural-network can memorize
such sample when the number of neurons is $\tilde{\Omega}(\frac{m^2}{d})$ \cite{oymak_towards_2019}.
Other results assume far worse dependence on the number of examples, requiring
the number of neurons to be polynomial in $m$ (refer to Table \ref{tbl:comparison}
for a comparison of the results).
In practice, on the other hand, a neural network needs only slightly more than
$\frac{m}{d}$ neurons to memorize a sample of size $m$ (observe the experiments in \cite{oymak_towards_2019}).

To understand why there is such a significant gap between theoretical and empirical results,
we briefly review the main theoretical works on ReLU networks.
Most theoretical results in this context rely on the concept of Random Features.
Random feature schemes are in fact two-layer neural-networks,
where the first layer is fixed (after random initialization), and the second is trained.
These ``networks'' have been shown to approximate various kernels,
proving to be more efficient than kernel methods \cite{rahimi2008random}.
While the original works on random features did not consider ReLU activations,
it has been shown that similar results can be given for many network architecture
and activation functions. The work of \cite{daniely2016toward} shows that when only
the last layer of a neural-network is trained, it can approximate functions from the
kernel space induced by the activation function and architecture.
However, when assuming that only the last layer is trained, the parameter
utilization is by definition very low. Indeed, observe that a one hidden-layer
network with $k$ hidden neurons and output in $\reals$,
has $dk$ parameter in the first layer but only $k$ parameters in the second.
Hence, training only the last layer is sub-optimal
(in terms of parameter utilization) by at least a factor of $d$.

In practice, however, all layers of the neural-network
are trained. To this end, there are many recent
works analyzing this typical setting, where gradient-descent updates all layers
of the network \cite{xie2016diverse,daniely2017sgd, du2018gradient,oymak2018overparameterized,allen2018learning,allen2018convergence,
arora2019fine, oymak_towards_2019, ma2019comparative, lee2019wide}.
While the details of each work vary, the key idea in all of these
works is the following: when the network is large enough, the weights of the network
change very little during the training process. Hence, training a neural-network
is ``almost'' a random features scheme, as the activation are governed completely
by their value upon initialization. 
Since in order to apply such argument
the neural-network is required to be rather large, the results obtained in this fashion
are also very far from being tight.

One approach for closing the gap between theory and practice is to try harder:
apply more complex theoretical tools, perform tedious analysis and hope to get improved
results for ReLU networks.
Another approach is to study simplified models, that are different than those used in
practice, but can nonetheless provide significant insights on
ReLU networks. A primary example for such simplified model is linear networks -
neural-networks with the linear activation function. Indeed, there is a growing 
body of work providing various results on optimization of linear networks,
showing different convergence properties \cite{saxe2013exact,kawaguchi2016deep,lu2017depth,nadav1,nadav2}.
While these are very far from neural-networks
used in practice, and in fact do no offer any improvement over simple linear classifiers,
they exhibit some phenomena that are also observed in ReLU networks.
Another example of a simplified model is networks with quadratic activation function
($\sigma(x) = x^2$) or polynomial activation.
Although such networks are not used in practice,
they are studied in theoretical works \cite{livni2014computational, mondelli2018connection, soltanolkotabi2019theoretical}.

Simplified models are attractive from a theoretical perspective, as they are
obviously simpler to analyze. However, it is often not clear whether the results
obtained for simple models are relevant for the cases that are of real interest.
Linear networks, for example, implement only linear functions and therefore cannot
account for learnability of complex non-linear functions learned by ReLU networks.
Networks with polynomial activations can implement only low-degree polynomials,
and therefore are very different from ReLU network, even from an expressivity
point-of-view.

In this work, we introduce a new simplified model that enjoys the best of both worlds:
it is simple to analyze, and yet maintains great similarity to ReLU networks.
This simple model arises from the observation that the output of a ReLU neuron
is a product of a linear function with a gating mechanism. That is,
we can write $[\vx^\top \vw]_{+} = (\1_{\vx^\top \vw \ge 0}) \cdot (\vx^\top \vw)$.
Notice that both the gate and the linear function share the same parameter $\vw$.
Our simplified model is in fact a \textit{generalization} of the ReLU neuron,
in which the gating and the linear function are determined by two \textbf{different}
parameters. This gives rise to a neural-network composed of Gated Linear Units
(GaLU network), where each unit is a function
$f_{\vw,\vu}(\vx) =  (\1_{\vx^\top \vu \ge 0}) \cdot (\vx^\top \vw)$.
Note that the gradient of this function with respect to the gate $\vu$ is always zero,
so we cannot use gradient-descent to learn the gates. Instead, these
gates are randomly initialized, and stay constant throughout the training process.

Since a GaLU network is a generalization of a ReLU network, its expressive power
is at least as good as that of a ReLU network. As noted,
other simple models are essentially weaker than ReLU networks in terms of expressivity.
On the other hand, GaLU networks are indeed simpler to analyze than ReLU networks, since
their gates remain fixed throughout the training process. Using this fact
allows us to give optimization and generalization results for GaLU networks,
that are much stronger than those available for ReLU networks.
Specifically, we show that for the memorization task mentioned above,
a GaLU network needs only $\tilde{\Omega}(\frac{m}{d})$ neurons,
which is essentially the minimal possible number of neurons needed for this task.
Furthermore, we prove generalization results for GaLU network that improve on
the equivalent results for ReLU networks. Finally, we show
that in some scenarios,
GaLU networks exhibit great similarity to ReLU networks.
All these results indicate that GaLU networks are a good simplified model for
ReLU networks, and we believe they can be used to provide further results
that will contribute to our understanding of ReLU networks.

As a final remark, it should be emphasized that we do not
claim that ReLU and GaLU networks are equivalent from the optimization
point of view. Indeed, in some problems, the fact that in ReLU
networks the weight vectors of the gate and linear part are shared
steers the optimization problem to a better direction. What we claim
is that GaLU networks are a simpler model, that often performs
similarly to ReLU networks and hence can shed light on the performance
of ReLU networks as well.

\iffalse
\section{Definitions and Notations}
\todo{decide if we want such section}
\begin{remark}
Unless otherwise noted, $\norm{\cdot}$ indicates the $\ell_2$ norm for vectors,
and the induced operator norm for matrices (i.e $\norm{A} = \sigma_{\max}(A)$).
\todo{Make sure that this is the convention in the paper}
\end{remark}

\todo{we sometimes use this, and sometimes $\1_{x \ge 0}$, we better decide on one
of them}. We denote $\sigma(x) = \begin{cases} 1 & x > 0 \\ 0 & x \le 0 \end{cases}$.

\todo{Define $L_S$, $L_{\mathcal{D}}$ according to the notations in the paper}

\fi

\section{GaLU Networks}

Consider a neuron with ReLU activation. It is a
function $f_\vw\left(\vx\right) : \reals^d \times \reals^d \to \reals$ such
that:
\begin{align*}
f_\vw \left(\vx \right)
  = \max\left\{\vx^\top \vw, 0\right\}
  = \left(\1_{\vx^\top \vw \ge 0}\right) \cdot \left( \vx^\top \vw \right) ~.
\end{align*}

The latter formulation demonstrates that the parameter vector $\vw$ plays two
roles in determining the value of the neuron. It decides whether the
output is $0$ or not: it acts as a \emph{filter} for some \emph{gating
mechanism}. It also determines the value of the neuron, assuming that the
neuron is active. In this role the parameter $\vw$ acts as the \emph{linear
weights} of the neuron.
It is not immediately clear why it makes sense for the two roles to be filled
by a single parameter. There are some intuitive explanations, and it is
partially motivated by neuroscience, but essentially the justification for
using ReLU neurons comes from the practical success of ReLU networks.

This work starts from the assumption that the connection between those two roles
doesn't have a strong theoretical justification. We propose, at least
tentatively, to consider a generalization of the ReLU neurons, that we call
\emph{GaLU} neurons (GaLU for ``Gated Linear Unit''). A GaLU neuron is a
function
$g_{\vw, \vu}\left(\vx\right): \reals^d \times \reals^d \times \reals^d \to \reals$
such that:
\begin{align*}
g_{\vw, \vu}\left(\vx\right)
  = \left(\1_{\vx^\top \vu \ge 0}\right) \cdot \left(\vx^\top \vw\right) ~.
\end{align*}

GaLU networks are networks built from GaLU neurons. Note that GaLU
is not, strictly speaking, an activation function: activation functions are
generally $\reals\to\reals$ functions that are composed with a linear function
to create a neuron. In this sense, GaLU breaks the common paradigm, but that
shouldn't be taken too seriously: gated units appeared in the deep learning
literature before.

GaLU neurons, and therefore GaLU networks, are at least as expressive as 
their ReLU counterparts, since $f_\vw = g_{\vw, \vw}$. So every expressivity
result on ReLU networks is immediately also an expressivity result on GaLU
networks. The expressive power is potentially much greater.

However, this shouldn't convince anyone that the research of GaLU networks is
of any relevance. To anyone who is familiar with deep learning practices, GaLU
networks should seem highly suspicious. The parameters $\vu$ of the networks
cannot be trained using gradient based optimization. As
$\nabla_\vu g_{\vw, \vu}\left(\vx\right)=\0$
at every point, attempting to use gradient based algorithm
would simply leave them intact. As gradient based algorithm are the common
optimization tool in deep learning, finding the optimal solution seems to be
completely hopeless.

In the following section we show that randomly initializing
the gates and fixing them throughout the optimization process is enough.
In other words, the gradient based optimization is only important for learning the
linear weights, while the random initialization gives the model enough expressive power.
In fact, for such training scheme we get optimization and generalization results
that are essentially stronger than current results that appear in the literature of ReLU networks.

\section{Theoretical Results for GaLU Networks}
Consider a GaLU network with a single hidden layer of $k$ neurons:
$\nN \left(\vx\right) = \sum_{j=1}^{k} \alpha_j g_{\vw_{j},\vu_{j}} \left(\vx\right) ~ .$
A convenient property of a GaLU neuron is that it is linear in the weights 
$w_j$, hence,
$\alpha_j g_{\vw_{j},\vu_{j}}\left(\vx\right) = g_{\alpha_j \vw_{j}, \vu_{j}}\left(\vx\right)$.
It means that the network can be rewritten as:
\begin{align*}
\nN\left(\vx\right)
  = \sum_{j=1}^{k} \alpha_j g_{\vw_{j},\vu_{j}}\left(\vx\right)
  = \sum_{j=1}^{k} g_{\alpha_j \vw_{j},\vu_{j}}\left(\vx\right) 
  = \sum_{j=1}^{k} g_{\tilde{\vw}_{j},\vu_{j}}\left(\vx\right)
\end{align*}

with $\tilde{\vw}_{j} = \alpha_j \vw_{j}$. Because we want to optimize over 
the weights $\vw_1, \dots, \vw_k, \alpha_1, \dots, \alpha_k$, we might as 
well optimize over the reparameterization
$\tilde{\vw}_{1}, \dots, \tilde{\vw}_{k}$ without losing expressive power. 
It means that in a GaLU network of this form, it is sufficient to 
train the \emph{first} layer of the network, as the readout layer adds 
nothing to the expressiveness of the network (as long all the weights are
non-zero).

The previous term can be further simplified:
%
%\begin{align*}
%\nN\left(\vx\right)
%  &= \sum_{j=1}^{k}g_{\vw_{j},\vu_{j}}(\vx) \\
%  &= \sum_{j=1}^{k}\1_{\vx^\top \vu_j \ge 0}\vx^{\top}\vw_{j} \\
%  &= \begin{bmatrix}
%       \1_{\vx^\top \vu_1 \ge 0} \vx^\top &
%       \1_{\vx^\top \vu_2 \ge 0} \vx^\top &
%       \dots &
%       \1_{\vx^\top \vu_k \ge 0} \vx^\top
%     \end{bmatrix}
%     \begin{bmatrix}
%       \vw_{1} \\
%       \vw_{2}\\
%       \vdots\\
%       \vw_{k}
%     \end{bmatrix} \\
%   &= \Phi_\mU \left(\vx\right)^\top \vw
%\end{align*}
\begin{align*}
\nN\left(\vx\right)
  = \sum_{j=1}^{k}g_{\vw_{j},\vu_{j}}(\vx) 
  = \sum_{j=1}^{k}\1_{\vx^\top \vu_j \ge 0}\vx^{\top}\vw_{j} 
= \Phi_\mU \left(\vx\right)^\top \vw
\end{align*}

where $\Phi_\mU \left(\vx\right)
  = \begin{bmatrix}
      \1_{\vx^\top \vu_1 \ge 0} \vx \\
      \1_{\vx^\top \vu_2 \ge 0} \vx \\
      \vdots \\
      \1_{\vx^\top \vu_k \ge 0} \vx
    \end{bmatrix}, \, 
\vw 
  = \begin{bmatrix}
      \vw_1 \\
      \vw_2 \\
      \vdots \\
      \vw_k
    \end{bmatrix},\ 
\mU
  = \begin{bmatrix}
      \vu_1 & \vu_2 & \dots & \vu_k
    \end{bmatrix} ~$.

So it turns out that a GaLU network is nothing more than a random non-linear 
transformation $\Phi_\mU : \reals^d \to \reals^{kd}$ and then a linear 
function. It immediately implies that for any convex loss function, it is a
convex problem to find the optimal solution. So for a single-layer GaLU network
with output in $\reals$ it is possible to find an optimal solution by this
reparameterization.

It still doesn't explain why running SGD on the natural
parameterization of the network should work: in the natural parameterization,
the problem is non-convex.
However, there are quite a few recent results on the ease of optimization of
linear networks. If we assume that the loss function is the squared loss,
we can use theorem 3 from \cite{zhu_global_2018}, and see that the objective
function has no spurious local minima and obeys the strict saddle property,
which essentially means that SGD converges to an optimal solution.

In the following section, we use this formulation of a GaLU network to prove
some strong results on optimization and generalization of such networks.
We show that training a GaLU network converges to a solution with zero training error,
when the number of parameters scales linearly (up to logarithmic factor) with the number
of examples. We then give generalization results depending on the induced kernel
space of the network.
%Finally, we analyze an explicit example of highly clustered pieceweise linear functions.

\subsection{Optimization Analysis}
\label{sec:shallow_optimization}
Fix some sample $S = \{(\vx_1, y_1), \dots, (\vx_m, y_m)\} \subset \reals^d \times \reals$,
and consider the optimization problem of learning a GaLU network minimizing the $\ell_2$ loss on the sample $S$:
\begin{align*}
\arg \min_{\mU, \vw} L_S(\mathcal{N}) = \arg \min_{\mU, \vw}
\sum_{i=1}^m (\mathcal{N}(\vx_i)-y_i)^2
= \arg \min_{\mU, \vw}
\sum_{i=1}^m (\Phi_{\mU}(\vx_i)^\top \vw-y_i)^2
\end{align*}

As noted, this is a convex optimization problem.
We can rewrite this optimization problem as follows:
let $\mX \in \reals^{m \times d}$ be the examples matrix
(each example is a row in $\mX$). Denote:
\[
\bar{\mX}^{(i)} = \left[
\begin{matrix}
(\1_{\vu_i^\top \vx_1 \ge 0}) \cdot \vx_1 \\ \vdots \\ (\1_{\vu_i^\top \vx_1 \ge 0}) \cdot \vx_m
\end{matrix} \right] \in \reals^{m \times d},
~\bar{\mX} = \left[ \bar{\mX}^{(1)} \dots \bar{\mX}^{(k)} \right] \in \reals^{m \times dk}
\]
We can then write the optimization problem as:
$\arg \min_{\mU, \vw} \norm{\bar{\mX} \vw - \vy}^2$.

Now, from standard results for linear regression, we know that if
$\rank(\bar{\mX}) = m$ (or alternatively, if the minimal singular value of 
$\bar{\mX}$ satisfies $\sigma_{\min}(\bar{\mX}) > 0$),
then the solution $\vw^* = \bar{\mX}^\top (\bar{\mX}\bar{\mX}^\top)^{-1} \vy$
achieves zero loss. Since for this convex problem, gradient-descent converges
to the optimal solution, it is enough to show that $\sigma_{\min}(\bar{\mX}) > 0$
to guarantee the convergence to zero loss solution.

Note that the matrix $\bar{\mX}$ depends on the examples $\mX$
and on the randomly initialized gates $\mU$. In general, we cannot guarantee
that it will have full row rank. If there are two identical examples in the sample,
then $\bar{\mX}$ will have two identical rows, and thus will not be full rank.
Similarly, if many of the gates in $\mU$ are similar, then we may have
dependence between columns in the matrix, which will also limit the rank.

To overcome this problem, we assume that the data is ``nice'' enough,
i.e - that it does not contain examples that are very similar.
Then, by initializing the gates from a normal distribution,
we can confirm that the matrix $\bar{\mX}$ will have full rank with high probability.
So throughout the paper we will assume $\vu_j \sim N(0,I_d)$.
To formalize our assumption on the data, we denote:
\[
\lambda(\mX) = \lambda_{\min}
\left(
\frac{1}{k}\mean{\vu_1, \dots, \vu_k \sim N(0,I_d)}{\bar{\mX} \bar{\mX}^\top}
\right)
\]

In our theoretical analysis, we assume that $\lambda(\mX) > 0$.
Note that this value depends only on the data, and not on the choice of gates.
We use the same notation as in \cite{oymak_towards_2019} (which gives
an equivalent definition of $\lambda(\mX)$),
and note that many other results for ReLU networks make the same assumption
(\cite{du2018gradient,arora2019fine}).
In the work of \cite{xie2016diverse}, the behavior of $\lambda(\mX)$ is studied,
and it is shown that typically, it is indeed strictly positive.
Given this assumption, we get that for a large enough GaLU network,
the matrix $\bar{\mX}$ is full rank with high probability:
\begin{lemma}
\label{lem:min_eig}
Assume $\lambda(\mX) > 0$ and fix $\delta > 0$.
If $k \ge \frac{8\norm{\mX}^2}{\lambda(\mX)}
\log(\frac{m}{\delta})$ then with probability at least $1-\delta$
we have: $\sigma_{\min}(\bar{\mX})^2 \ge \frac{k}{2} \lambda(\mX)$.
\end{lemma}

To apply this lemma, the number of neurons $k$ needs to be on the order of
$\frac{\norm{\mX}^2}{\lambda(\mX)}$ (up to logarithmic factors).
In \cite{oymak_towards_2019} is is shown that when the data is Gaussian,
we get that w.h.p. $\norm{\mX} = O(\sqrt{\frac{m}{d}})$
and that $\lambda(\mX)$ behaves like a constant.
Therefore, the number of neurons in this case is
$\tilde{\Omega}{(\frac{m}{d})}$:

\begin{lemma}
\label{lem:min_eig_gaussian}
Assume $\vx_i \overset{\textnormal{i.i.d.}}{\sim} \textnormal{Uni}\left(\mathbb{S}^{d-1}\right)$
and assume $\vu_i \sim \mathcal{N}(0,1)$. Then there exist
$\gamma_1, \gamma_2, c_1, c_2>0$ such that for $d \le m \le c_2d^2$, if
$k \ge \frac{64\pi}{c_1^2} \frac{m}{d} \log(\frac{m}{{\delta}})$, then
we have $\sigma_{\min}(\bar{\mX}) \ge \frac{\sqrt{k} c_1}{\sqrt{4\pi}} > 0$
with probability of at least
$1-me^{-\gamma_1 \sqrt{m}} - \frac{1}{m} - (2m+1)e^{-\gamma_2 d}-\delta$.
\end{lemma}

Notice that the number of trainable parameters in a GaLU network is $kd$.
Therefore, the number parameters required for our result to hold
scales linearly (up to logarithmic factors) with the number of examples.
Generally speaking, to fit an arbitrary sample we need the number of parameters
to be at least the number of examples, so in this sense our result is almost optimal.
To the best of our knowledge, this is the first result that shows convergence to zero loss,
when the number of parameters scales only linearly with the number of examples.
For comparison, the best result for a ReLU network requires that the number
of parameters scales with $m^2$.
Table \ref{tbl:comparison} shows a comparison between our result and previous optimization results that are directly comparable.

\begin{table}
  \caption{Comparison of network sizes in different optimization results.}
  \label{tbl:comparison}
  \centering
  \begin{tabular}{lll}
    \toprule
    Paper     & Activation     & Network Size \\
    \midrule
%    Daniely \cite{daniely2017sgd}  & ReLU or bounded activation  & \todo{ask amit...}  \\
    Du \cite{du2018gradient}  & ReLU & $\Omega(m^6)$  \\
    Arora \cite{arora2019fine}  & ReLU & $\Omega(m^7)$  \\
    Oymak \cite{oymak_towards_2019}  & Bounded derivatives & $\tilde{\Omega}(\frac{m^2}{d})$  \\
    \midrule
    Ours          & GaLU & $\tilde{\Omega}(\frac{m}{d})$ \\
    \bottomrule
  \end{tabular}
\end{table}

To finish the optimization analysis, we turn to analyzing the behavior of gradient-descent
when optimizing a GaLU network.
We showed that a very mild over-parametrization is sufficient for $\bar{\mX}$
to be of rank $m$. Now, in this case, from standard results from convex optimization we get that gradient descent converges linearly to $\vw^*$:

\begin{theorem}
\label{thm:shallow_convergence}
Assume $\lambda(\mX) > 0$ and fix $\delta > 0, \epsilon > 0$. 
Let $k \ge \frac{8\norm{\mX}^2}{\lambda(\mX)}
\log(\frac{m}{\delta})$, and assume we initialize a GaLU network with $k$ neurons.
Fix $\eta = \frac{m}{k \norm{\mX}^2}$.
Then with probability at least $1-\delta$ on the initialization of the gates,
after $t \ge \frac{2\norm{\mX}^2}{\lambda(\mX)}
\log(\frac{k \norm{\mX}^2 \norm{\vw_0-\vw^*}^2}{m\epsilon})$
iterations of gradient-descent
with step size $\eta$, the value of the loss function is bounded by $\epsilon$.
\end{theorem}

While the above analysis applies for cases where the number of parameters is
larger than the number of examples, it is also interesting to observe situations
where this is not the case. In these cases, we cannot guarantee convergence to zero
loss without further assumptions on the labels. On the other hand,
we can still give an estimation of the loss, using the results we have shown so far.
The following theorem estimates the loss achieved by a GaLU network,
when the number of parameters is not necessarily large enough to guarantee zero loss:

\begin{theorem}
\label{theorem:rank}
Assume that $\ry_1, \dots , \ry_m \sim N(0, 1)$. 
Define the expected squared loss on the training set, for
weights $\vw$, as $L_S(\vw)$. Then we have:
$\mathbb{E} [\min_\vw L_S(\vw)] = 1 - \frac{ \rank \left(\bar{\mX}\right)}{m}$.
\end{theorem}

Now, when there are not enough parameters, we get that $\rank(\bar{\mX}) \simeq kd$,
so the loss behaves like $1-\frac{dk}{m}$. Therefore, 
we get a characterization of the loss which holds in the under-parametrized case.
This is shown formally in the following Corollary:

\begin{corollary}
\label{crl:shallow_under_parametrized}
There exist some absolute constants
$\gamma_1, \gamma_2, c_1, c_2>0$ such that the following holds:
Fix $\delta > 0$ and $k > 0$, denote $m' = \lfloor kd \frac{c_1^2}{64 \pi} \log^{-1} (\frac{c_2 d^2}{\delta})
\rfloor$ and assume $d \le m' \le c_2 d^2$.
Assume $\vx_i \overset{\textnormal{i.i.d.}}{\sim} \textnormal{Uni}\left(\mathbb{S}^{d-1}\right), y_i \sim N(0,1)$
and assume $\vu_i \sim \mathcal{N}(0,1)$.
Then with probability of at least
$1-m'e^{-\gamma_1 \sqrt{m'}} - \frac{1}{m'} - (2m'+1)e^{-\gamma_2 d}-\delta$
we have: $\mean{}{\min_{\vw} L_S(\vw)} \le 1 - \frac{m'}{m}$.
\end{corollary}

In this section we considered a pure memorization task,
where the labels may be independent of the input examples.
While this is an interesting task from a theoretical point of view,
it is not immediately clear why this result is relevant in practice.
However, we note that in many cases memorization
is an important tool in solving various complex problems.
For example, when the data is highly clustered around a few cluster
centers, memorizing the labels of the cluster centers is a simple technique
that is often used in practice. We show that our results can also be
applied for highly clustered data. In this case we require
that the number of neurons scales with the number of cluster centers,
and does not depend on the number of examples.
For lack of space, we leave this analysis to appendix \ref{sec:highly_clustered}.

\subsection{Generalization in the Over-Parametrized Case}
\label{sec:shallow_generalization}
In this section, we give a generalization bound for learning GaLU networks.
Before we do so, let us review the main approach used for analyzing ReLU networks:
\begin{enumerate}
\item Define a kernel associated with the ReLU network, and observe functions
with large-margin in the induced Hilbert space. These functions are learnable
via standard kernel learning.
\item Show that the defined kernel can be approximated using a random-features
scheme. Hence, large-margin functions can be learned using random features.
\item Show that when training a large enough ReLU network, the weights stay
close to the initialization point. Since the weights are randomly initialized,
this shows that a ReLU network essentially implements a random-features scheme.
\end{enumerate}

We take a similar approach when analyzing the generalization of GaLU networks.
We study the kernel associated with the GaLU network,
and show that a GaLU network can learn functions from the Hilbert space induced
by this kernel. In fact, we observe that the kernel of the GaLU network
is the same kernel used for the analysis of ReLU network.
That said, notice that there is a crucial difference between the analysis of
GaLU networks and that of ReLU networks. 
While for the analysis of ReLU networks it is essential to show that the network's weights
stay close to their initial value, this property is not required for GaLU networks.
Since the gates of GaLU stay fixed through the entire training process,
the non-linear part of the network is defined upon initialization, and does not change.
Therefore, step 3 in the scheme above becomes trivial for GaLU networks.

We begin with a few definitions.
To simplify the analysis, we consider the normalized GaLU network:
\[
\mathcal{N}(\vx) = \frac{1}{\sqrt{k}} \sum_{j=1}^k g_{\vw_j,\vu_j}(\vx)
= \frac{1}{\sqrt{k}} \Phi_{\mU}(\vx)^\top \vw
\]

We define the following kernel:
\[
\kappa\left(\vx,\vy\right) =
\mean{\vu \sim N(0,I_d)}{(\1_{\vu^\top \vx \ge 0}) \cdot (\1_{\vu^\top \vy \ge 0})
\inner{\vx,\vy}}
= \left( \frac{1}{2}-\frac{\arccos\inner{\vx,\vy}}{2\pi} \right) \inner{\vx,\vy}
\]

This is the same kernel associated with ReLU networks in previous works.
Notice that we have the following relation between the kernel $\kappa$ and
the GaLU neurons:
\[
\mean{\mU}{\inner{\frac{1}{\sqrt{k}}\Phi_{\mU}(\vx), \frac{1}{\sqrt{k}} \Phi_{\mU}(\vy)}}
= \mean{\mU}{\frac{1}{k} \sum_{j=1}^k (\1_{\vu_j^\top \vx \ge 0}) \cdot 
(\1_{\vu_j^\top \vy \ge 0}) \inner{\vx,\vy}} = \kappa(\vx,\vy)
\]

So we can think of a GaLU network as a random-features scheme approximating
the kernel $\kappa$. Let $\mathcal{H}_{\kappa}$ be the RKHS induced by this kernel,
and denote $\norm{\cdot}_{\kappa}$ the norm of $\mathcal{H}_\kappa$.
We denote
$\mathcal{B}_{\kappa}(M) = \{f \in \mathcal{H}_{\kappa} ~:~ \norm{f}_{\kappa} \le M\}$,
the set of function in $\mathcal{H}_{\kappa}$ with norm bounded by $M$.
Let $\mathcal{D}$ be a distribution over $\mathcal{X} \times [-1,1]$ that is
separable by $\mathcal{B}_{\kappa}(M)$, \emph{i.e.}, there is
$f^* \in \mathcal{B}_{\kappa}(M)$, such that if $(\vx,y) \sim \mathcal{D}$ then
$y = f^*(\vx)$ with probability 1. Then we have the following generalization bound:

\begin{theorem}
\label{thm:shallow_generalization}
Assume $\lambda(\mX) > 0$, and fix $\delta > 0$.
Let $k \ge (\frac{m}{M^2 \lambda(\mX)} + 1)^2 \frac{32\norm{\mX}^4}{\lambda(\mX)^2} \log(m/\delta)$,
then with probability at least $1-2\delta$,
the generalization error of the GaLU network is bounded by 
$C \left( \frac{2M^2 \log^3m + (2M^2 + \sqrt{2}M)
\log(1/\delta)}{m} \right)$.
\end{theorem}

Compare this result to the generalization bound presented in the recent work
by \cite{arora2019fine}. In this result, generalization bound is obtained
when the network size grows with $m^7$, while our bound requires a more modest
(yet admittedly large) dependence on the number of examples.
Furthermore, our generalization bound decays with $\frac{1}{m}$, while the bound
shown in \cite{arora2019fine} decays with $\frac{1}{\sqrt{m}}$.

\section{Relation to ReLU}
\label{sec:relu_relation}
So far, we showed various results analyzing optimization and generalization
of GaLU networks. These results depend on some convenient properties of GaLU
neurons, that make their analysis much simpler then their ReLU counterparts.
However, since ReLU networks are extremely popular, and achieve remarkable performance
empirically, it would be beneficial to account for the relation between 
GaLU and ReLU networks. In this section, we aim to understand to what extent results
shown for GaLU networks can be applied for ReLU, and vice-versa.
As in any algorithmic research field, there are two types of results on ReLU networks:
positive results, that show cases where ReLU networks succeed in a given task,
and negative results, that present interesting failure cases.
To this end, we wish to show that for both positive and negative results,
GaLU networks are a good proxy for ReLU networks.
We show two results in this context.
First, we observe that failure cases of GaLU, i.e - cases where the optimization
of a GaLU network fails upon initialization, immediately imply that a ReLU
network will fail on the same data, and vice-versa. Second, we show that in some cases,
the best GaLU network with fixed random gates is competitive with the best
ReLU network.

We begin by reviewing some notations that will allow us to compare GaLU networks
to ReLU networks. Given a set of weights
$\mW = \{\vw_1, \dots, \vw_k\}$, a set of gates $\mU = \{\vu_1, \dots, \vu_k\}$
and a set of scalars $\alpha = \{\alpha_1, \dots, \alpha_k\}$, a normalized
GaLU network is defined as:
\[
\mathcal{N}_{\mW, \mU, \alpha}^{G}(\vx)
= \frac{1}{\sqrt{k}} \sum_{i=1}^k \alpha_i g_{\vw_i, \vu_i}(\vx)
\]
We can define similarly the equivalent ReLU network:
\[
\mathcal{N}_{\mU, \alpha}^{R}(\vx)
= \frac{1}{\sqrt{k}} \sum_{i=1}^k \alpha_i f_{\vu_i}(\vx)
= \frac{1}{\sqrt{k}} \sum_{i=1}^k \alpha_i g_{\vu_i, \vu_i}(\vx)
\]

\subsection{Failure of GaLU vs. Failure of ReLU}
In this part we use the hinge loss $\ell(y,\hat{y}) = \max \{1-y\hat{y},0\}$,
instead of the square loss, to simplify the analysis.
Notice that the optimization results in section \ref{sec:shallow_optimization}
depend on the data being ``nice'' enough (which is captured by the assumption
that $\lambda(\mX) > 0$). However, we might encounter extreme cases where the data
doesn't behave ``nicely''. These cases can cause the optimization to fail,
and achieve large train loss. In fact, in some extreme cases the failure may happen
upon initialization. In these cases, the gradient will be very small with high
probability upon the initialization.
Refer to \cite{shamir2018distribution,shalev2017failures}
for examples of such cases.
We show that in these failure cases,
the behavior of GaLU and ReLU are similar: GaLU fails if and only if ReLU fails.

\begin{theorem}
\label{thm:failure_equivalence}
Let $\mathcal{N}^R_{\mU, \alpha}$ be a ReLU network,
and let $\mathcal{N}^G_{\mW, \mU, \alpha}$ be a GaLU network,
both initialized such that $\mathcal{N}^R_{\mU, \alpha}(B_1),
\mathcal{N}^G_{\mW, \mU, \alpha}(B_1) \subseteq [-1,1]$.
Then $\norm{\frac{\partial}{\partial \mW} L_S(\mathcal{N}^G_{\mW, \mU, \alpha})}
\le \epsilon$ with probability $1-\delta$ upon initialization if and only if 
$\norm{\frac{\partial}{\partial \mU} L_S(\mathcal{N}^R_{\mU, \alpha})} \le \epsilon$
with probability $1-\delta$ upon initialization.
\end{theorem}

\subsection{GaLU Networks are Competitive with Large ReLU Networks}

As mentioned, various previous results show that when training a large ReLU network,
gradient-descent reaches a stationary point with zero loss with high probability
\cite{xie2016diverse,daniely2017sgd, du2018gradient,oymak2018overparameterized,allen2018learning,allen2018convergence,
arora2019fine, oymak_towards_2019, ma2019comparative, lee2019wide}.
All of these results rely on the key observation that when the network is large enough, the
weights of the network barely change from their initial value. In this part we show that
if this is the case, \emph{i.e.} if the value of the weights of the ReLU network
changes very little, then the best GaLU network (with randomly initialized gates)
achieves loss that is competitive with the best ReLU network.
 
To formalize this, let $\mathcal{D}$ be a distribution over
$\mathcal{X} \times \mathcal{Y}$ and assume we initialize
$\vu_1, \dots, \vu_k \sim N(0,I_d)$.
Fix some $L$-Lipschitz loss $\ell : \reals \times \mathcal{Y} \to \reals$,
and observe the loss on the distribution
$L_{\mathcal{D}}(f) = \mean{(\vx,y) \sim \mathcal{D}}{\ell(f(x),y)}$.
Let $\mathcal{N}^G_{\mW^*, \mU, \alpha^*}$ be the optimal GaLU network with respect to $L_{\mathcal{D}}$
(with gates $\vu_1, \dots, \vu_k$ fixed),
so $\mW^*, \alpha^* = \arg \min_{\mW, \alpha}
L_{\mathcal{D}}(\mathcal{N}^G_{\mW, \mU, \alpha})$.
Let $\mathcal{N}^R_{\mU^*, \alpha^{**}}$ be the optimal
ReLU network with respect to $L_{\mathcal{D}}$
satisfying that $\norm{\vu_i^* - \vu_i} \le \epsilon$ for all $i \in [k]$
(small distance from initialization),
so $\mU^*, \alpha^{**} = \arg \min_{\mV, \alpha ~s.t~ \norm{\vv_i - \vu_i} \le \epsilon} L_{\mathcal{D}}
(\mathcal{N}_{\mV, \alpha}^R)$. Then we get:
\begin{theorem}
\label{thm:relu_similarity}
Fix $\delta > 0$, let $k \ge \frac{\pi}{\sqrt{6}d \epsilon^2} \left( \log(2/\delta)
+ d\log(3/\epsilon) \right)$,
and we assume $d > \log(2k/\delta)$.
Then with probability at least $1-\delta$,
we have:
\[
L_{\mathcal{D}}(\mathcal{N}^G_{\mW^*, \mU, \alpha^*}) \le 
L_{\mathcal{D}}(\mathcal{N}^R_{\mU^*, \alpha^{**}}) + L 
\sqrt{\frac{5\sqrt{3d}\epsilon}{\sqrt{2\pi}}}
\cdot \max_i \norm{\alpha_i^{**} \vu_i^*}
\]
\end{theorem}

This result means that GaLU networks with randomly initialized gates are
competitive with ReLU networks with small distance from initialization. Therefore, GaLU networks are indeed a good simplified model
for ReLU networks, when the distance from initialization is small.

%
%\begin{theorem} \label{thm:relu_similarity}
%Fix $\delta > 0$, let $k \ge \frac{\pi}{\sqrt{6}d \epsilon^2} \left( \log(2/\delta)
%+ d\log(3/\epsilon) \right)$,
%and we assume $d > \log(2k/\delta)$. Fix a sample
%$S \subseteq (\mathcal{X} \times \mathcal{Y})^m$.
%Assume we initialize a ReLU network with weights
%$\vu^{(0)}_1, \dots, \vu^{(0)}_k \sim \mathcal{N}(0,I_d)$.
%Denote $\vu^{(t)}_j$ the value of the $j$-th weight vector after $t$ iterations
%of gradient-descent. Denote
%$f_t(\vx) = \frac{1}{\sqrt{k}} \sum_{j=1}^k [\vu_j^{(t)} \cdot \vx]_{+}$,
%the ReLU network after $t$ iterations.
%Denote $g_{\vw^*, \vu}(\vx) =
%\frac{1}{\sqrt{k}} \sum_{j=1}^k \1_{\vu_j^\top \vx \ge 0} \vw^*_j \cdot \vx$,
%the normalized GaLU network where $\vu = \vu^{(0)}$ and $\vw^*$ is the optimal weights
%for the GaLU network on the sample $S$.
%Then if gradient-descent reaches a stationary point
%after $t$ iterations when training the ReLU network,
%such that for every $j \in [k]$ we have $\norm{\vu^{(t)}_j - \vu^{(0)}_j} \le \epsilon$,
%then we get that w.p at least $1-\delta$:
%\[
%\norm{f_t - g_{\vw^*,\vu}}_\infty \le \sqrt{\frac{5B\epsilon}{\sqrt{2\pi}}} (\norm{\vw^*} + C_{m,d,k})
%\]
%For some $C_{m,d,k}$ depending on $m,d,k$.
%\end{theorem}

\section{Discussion}
In this paper we introduced a new neural-network model - the GaLU network.
Since optimization of a GaLU network is a convex problem,
these networks allow us to get strong theoretical results with much simpler tools.
Indeed, we showed theoretical results for GaLU networks that are significantly better
than equivalent results in the literature of ReLU networks.
Furthermore, since current analysis of ReLU networks assumes that the
weights of the network stay close to their initial value, we note that
in some sense current ReLU analysis is implicitly an analysis of GaLU networks.

However, we do not claim that GaLU networks fully capture the behavior of ReLU networks,
nor do we claim that they are a preferable model to use in practice.
Indeed, we perform various experiments, covering cases where the behavior of GaLU and ReLU
networks is similar, but also cases where they differ.
Due to the lack of space, these experiments are detailed in appendix \ref{sec:experiments}.
What we do claim is that a GaLU network is a better simplified model,
compared to other simplified
models that appear in the literature, such as linear networks or networks with polynomial
activation. These simplified models allow theoretical research to gain insights
on various aspects of neural-networks, and we believe that GaLU networks would prove
to be another useful tool in the theoretician's toolbox.

Finally, we note that the scope of this work is limited only to the analysis
of one-hidden layer networks with output in $\reals$. While this is a rich research
area, there is still much more to say about neural-networks in general.
Specifically, the analysis of shallow networks with vector-valued output,
as well as the research of deep networks and convolutional networks,
is not covered in this paper. We leave these promising research directions
to future work.

\paragraph{Acknowledgements:} This research is supported by the European Research Council (TheoryDL project).

\bibliography{bib}
\bibliographystyle{plain}

\clearpage
\appendix

\section{Proofs of section \ref{sec:shallow_optimization}}
\begin{proof} of Lemma \ref{lem:min_eig}
We use the following notations:
\[
\mH^{(i)} = \bar{\mX}^{(i)} (\bar{\mX}^{(i)})^\top
~ ; ~
\mH = \bar{\mX} \bar{\mX}^{\top} = \sum_{i=1}^k \mH^{(i)}
\]

Notice that $\lambda(\mX) = \lambda_{min}(\mean{}{\mH^{(i)}})$. 
Denote $R := \norm{\mX}^2$, and observe that we have:
$\lambda_{\max}(\mH^{(i)}) \le \norm{\mX}^2 = R$, so $\mH^{(i)}$ are i.i.d.
random postive semi-definite self-adjoint matrices with bounded norm. Notice
that
$\mu_{\min} := \lambda_{\min}( \sum_{i=1}^k \mean{}{\mH^{(i)}}) = k \lambda(\mX)$.
Now, we can use matrix Chernoff bound (\cite{noauthor_matrix_2018}) and get that:

\begin{align*}
\prob{}{\lambda_{\min}(\sum_{i=1}^k \mH^{(i)}) \le (1-\epsilon)\mu_{\min}}
\le m \cdot \left[\frac{e^{-\epsilon}}{(1-\epsilon)^{1-\epsilon}} \right]^{\mu_{min}/R}
= m \cdot \left[\frac{e^{-\epsilon}}{(1-\epsilon)^{1-\epsilon}} \right]^{k \lambda(\mX)/R}
\end{align*}

Now, if we take $\epsilon = \frac{1}{2}$, we get:

\begin{align*}
\prob{}{\lambda_{\min}(\mH) \le \frac{k}{2}\lambda(\mX)} = 
\prob{}{\lambda_{\min}(\sum_{i=1}^k \mH^{(i)}) \le \frac{k}{2}\lambda(\mX)}
&\le m \cdot \left(\frac{e}{2}\right)^{-k\frac{\lambda(\mX)}{2R}} \le \delta
\end{align*}

\end{proof}

\begin{proof} of Lemma \ref{lem:min_eig_gaussian}.
A recent work gives the following bound on $\lambda(\mX)$ (Lemma 6.4 in
\cite{oymak_towards_2019}):
\[
\lambda(\mX) \ge \frac{1}{2\pi} \sigma_{\min}^2(\mX \star \mX)
\]
Where $\mX \star \mX$ is the Khatri-Rao product.

Following a similar proof to Corollary 2.2 in \cite{oymak_towards_2019},
we have:
\[
\norm{\mX} \le 2\sqrt{\frac{m}{d}}
\]
with probability of at least $1-e^{-\gamma_2 d}$.
We also have:
\[
\sigma_{\min}(\mX \star \mX) \ge c_1
\]
with probability of at least
$1-ne^{-\gamma_1 \sqrt{m}} - \frac{1}{m} - 2me^{-\gamma_2 d}$.
Assuming that both of these hold, we get from what we have shown:

\begin{align*}
\prob{}{\lambda_{\min}(\mH) \le \frac{kc_1^2}{4\pi}}
&\le m \cdot \left(\frac{e}{2} \right)^{-k\frac{\lambda(\mX)}{2R}}\\
&\le m \cdot \left(\frac{e}{2} \right)^{-\frac{c_1^2}{16\pi} \frac{kd}{m}}
\le \delta
\end{align*}

Observing that $\sigma_{\min}(\bar{\mX}) = \sqrt{\lambda_{\min}(\mH)}$
and using union bound completes the proof.
\end{proof}

\begin{proof} of Theorem \ref{thm:shallow_convergence}.
Denote $\mH = \bar{\mX} \bar{\mX}^\top$ and $\mH^{(i)} =
\bar{\mX}^{(i)} (\bar{\mX}^{(i)})^\top$.
Now, assuming $\rank\left(\bar{\mX}\right) = m$, observe the
objective of the optimization of the GaLU network. From what we developed
previously, this objective is given by:
\[
F(\vw) = \frac{1}{2m} \norm{\bar{\mX}\vw-\vy}^2
= \frac{1}{2m}(\vw^\top \mH \vw - 2\vy^\top \bar{\mX}\vw + \norm{\vy}^2)
\]
Since $\bar{\mX}$ is full-rank, we can define the optimum of $F$ by $\vw^* = \bar{\mX}^\top \mH^{-1} \vy$, and we get: $F(\vw^*) = 0$.

Notice that $\lambda_{\max}(\mH) = \lambda_{\max}(\sum_{i=1}^k \mH^{(i)})
\le k \norm{\mX}^2$. From \ref{lem:min_eig}, with probability at least $1-\delta$
we have: $\lambda_{\min}(\mH) \ge \frac{k}{2} \lambda(\mX)$.
Therefore, applying Theorem 6.3 in \cite{hardt_ee227c:_2018}
gives:
\[
\norm{\vw_t - \vw^*}^2 \le \exp \left(-\frac{t\lambda(\mX)}{2\norm{\mX}^2}\right)
\norm{\vw_0-\vw^*}^2
\]
Now, we have $\nabla^2 F(\vw) = \frac{1}{m}\mH$, so:
\begin{align*}
\norm{\nabla F(\vw_t)} &= \norm{\nabla F(\vw_t) - \nabla F(\vw^*)} \\
&\le \norm{\nabla^2F(\vw)} \norm{\vw_t - \vw^*} \\
&= \frac{\norm{\mH}}{m} \norm{\vw_t - \vw^*} \\
&\le \frac{k \norm{\mX}^2}{m} \norm{\vw_t - \vw^*}
\end{align*}
Using the fact that $F$ is convex, we get:
\[
F(\vw_t) = |F(\vw_t) - F(\vw^*)|
\le \norm{\nabla F(\vw_t)} \norm{\vw_t - \vw^*} \le 
\frac{k \norm{\mX}^2}{m} \norm{\vw_t - \vw^*}^2
\]
Using what we previously showed, we get that w.p at least $1-\delta$ we have:
\[
F(\vw_t) \le \exp \left(-\frac{t\lambda(\mX)}{2\norm{\mX}^2}\right)
\frac{k \norm{\mX}^2}{m} 
\norm{\vw_0-\vw^*}^2 \le \epsilon
\]
\end{proof}

\begin{proof} of Theorem \ref{theorem:rank}.
Every vector $\vy = (\ry_1, \dots , \ry_m) \in \reals^m$ can be decomposed to
a sum $\vy = \mathbf{a} + \mathbf{b}$ where $\mathbf{a}$ is in the span of the
columns of $\bar{\mX}$ and $\mathbf{b}$ is in the null space of $\bar{\mX}$.
It follows that $\min_\vw L_S(w) = \|\mathbf{b}\|^2/m$. The claim follows
because if $\vy \sim N(0,I_m)$ then the expected value of
$\|\mathbf{b}\|^2$ is $m - \rank\left(\bar{\mX}\right)$. 
\end{proof}

\begin{proof} of Corollary \ref{crl:shallow_under_parametrized}.
Observe the sub-sample $S' \subseteq S$, which is simply the first $m'$ examples
from $S$. Denote $\mX' \in \reals^{m' \times d}$ the corresponding sub-matrix of $\mX$, and $\bar{\mX}' \in \reals^{m' \times dk}$  the corresponding sub-matrix of $\bar{\mX}$.
Then, from Lemma \ref{lem:min_eig_gaussian}, with probability at least
$1-m'e^{-\gamma_1 \sqrt{m'}} - \frac{1}{m'} - (2m'+1)e^{-\gamma_2 d}-\delta$,
the matrix $\bar{\mX}'$ has maximal rank, so $\rank \bar{\mX}' = m'$.
Therefore, it must hold that $\rank \bar{\mX} \ge m'$, so the result
follows from Theorem \ref{theorem:rank}.
\end{proof}

\section{Highly Clustered Piecewise Linear Data}
\label{sec:highly_clustered}
In the optimization analysis presented in section \ref{sec:shallow_optimization},
we saw that GaLU networks can achieve zero training loss when the number of parameters
grows with the number of examples. However, in practice neural-networks
can achieve low train error with relatively small amount of parameters.
To account for this gap, observe that in our optimization analysis we did not depend
on the value of the labels. That is, the same analysis can be applied for random labels
and for labels that depend on the input examples.
Naturally, we would like to show that when the labels depend on the inputs,
we can get better guarantees from an optimization point of view.
In this section, we analyze a model where the data is sampled from a
distribution over $n$ clusters, such that on each cluster the label is generated
by a distinct linear function. In such case, we show that to reach zero loss,
the number of neurons in the network depends only on the number of clusters,
with no dependency on the number of examples. This can potentially give much better
bounds on the required network width under this model.

We start by formalizing our model. We are going to consider a distribution
that is very clustered around $n$ cluster centers, and that
within each cluster, the label $\ry$ is a linear function of the input $\rvx$.
Fix $n \in \naturals$ to be the number of clusters, $r \in \reals$
the radius of each cluster, and $n$ linear transformations
$\ell_1, \dots, \ell_n \in \reals^d$.
Let $\vv_1, \dots, \vv_n \in \mathbb{S}^{d-1}$ be $n$ cluster centers.
Let $\mH \in \reals^{n\times n}$ be such that $\mH_{ij}=\frac{1}{2}-\frac{\arccos\left(\vv_i^\top \vv_j\right)}{2\pi}$,
and denote $\mu=\lambda_{\text{min}}\left(\mH\right)$. We shall assume $\mu > 0$
(and we will soon justify this assumption). Pick $\delta > 0$ and
$k \ge \frac {8n}{\mu} \log \left(\frac{n}{\delta} \right)$. Denote $r = \frac{\delta}{nk\sqrt{d}}$.

Define the distribution $\dD$ over $\mathbb{S}^{d-1}\times\reals$ by the following random process. First, pick $q \sim \mathcal{Q}$
where $\mathcal{Q}$ is some distribution over $\left[n\right]$. Then, pick $\vx \sim \dD_q$, where $\dD_q$ is a distribution over
$\mathbb{S}^{d-1}$ such that $\Pr\left(\norm{\vx - \vv_q}_2 > r\right) = 0$. Finally, return $\left(\vx, \vx^\top \ell_q\right)$.

For this model, we get much better results than in the general case.
Specifically, we show that when the number of neurons grows with the number
of \textbf{clusters}, a GaLU network achieves zero loss. Notice that 
in the previous results, we required that the number of parameters
grows with the number of examples, which typically can be much larger
than the number of cluster centers.
This is captured in the following theorem:

\begin{theorem}
\label{thm:highly_clustered}
Pick $\epsilon > 0$, and set $m = c \frac{nd \log \left(\frac{1}{\epsilon}\right) + \log\left(\frac{1}{\delta}\right)}{\epsilon}$
($c$ is a global constant). Let $\mW^*$ be the result of training a GaLU network with $k$ neurons on an i.i.d. sample from $\dD$.
Then, with probability $\ge 1 - 3\delta$, the training loss of the network on the sample is $0$, and the test loss is $\le \epsilon$.
\end{theorem}

Note that from the previous lemma, we get that the value of $k$ 
is governed by $\frac{n}{\mu}$. The value of $\mu$ depends only on the
choice of the cluster centers $\vv_1, \dots, \vv_n$, and we would like
to show that it is typically not too small. In fact, we will show that
when the dimension is large enough, namely $d = \Omega(n^2)$,
and when $\vv_i$-s are chosen randomly,
then $\mu$ is a constant.

\begin{lemma}
\label{lem:mu_min}
Fix $\delta > 0$. Assume $d \ge \frac{n^2}{2} \log\left(\frac{2n^2}{\delta}\right)$,
and assume we choose $\vv_i \sim Uni(\{\pm \frac{1}{\sqrt{d}}\}^{d})$.
Then with probability at least $1-\delta$ we have that $\mu \ge \frac{1}{8}$.
\end{lemma}
\subsection{Proof of Theorem \ref{thm:highly_clustered}}

The theorem follows from the following deterministic claim. Let $m_1,\dots,m_n \in \natural$ be $n$ cluster sizes,
and for every $i \in \left[n\right]$ let $S_i = \left\{ \left( \vx_{ip}, y_ip \right) \right\}_{p=1}^{m_i}$ be such that
for every $p \in \left[m_i\right]$, $\norm{\vx_{ip} - \vv_i} < r$ and $y_{ip} = \vx_{ip}^\top \ell_i$. Define
$S = \dot{\bigcup}_{i=1}^{n} S_i$. In addition, pick $q \in \left[n\right]$ and $\tilde{\vx}, \tilde{y}$ such that:
\begin{enumerate}
  \item $\tilde{\vx} \in \text{span} \left\{\vx_qp\right\}_{p=1}^{m_q}$.
  \item $\norm{\tilde{\vx} - \vv_q} < r$.
  \item $\tilde{y} = \tilde{\vx}^\top \ell_q$.
\end{enumerate}
\begin{theorem}
W.p. $\ge 1 - 2\delta$ over the choice of gates, there is an exact solution when training a GaLU network with $k$ neurons on $S$.
Moreover, any such solution would correctly predict the example $\left(\tilde{\vx},\tilde{y}\right)$.
\end{theorem}

Let $\vu_1,\dots,\vu_k\overset{\text{i.i.d.}}{\sim}\text{Uni}\left(\mathbb{S}^{d-1}\right)$ be the gates of the network.
Let $\mA=\left[a_{ij}\right]\in\reals^{n\times k}$ be such that $a_{ij}=\1_{\vu_j^{\top}\vv_i \ge 0}$. For every $i\in\left[n\right]$,
Let $\mX_i=\begin{bmatrix}\vx_{i1}^\top\\
\vx_{i2}^\top\\
\vdots\\
\vx_{im_i}^\top
\end{bmatrix}$.

\begin{lemma}
With probability of at least $1-\delta$, $\rank \left(\mA\right) = n$.
\end{lemma}

\begin{proof}
We shall show the stronger claim $\sigma_{\min}(\mA)^2 > \frac{k}{2} \mu$.
Denote $\mB_i = \mA_i^\top \mA_i \in \{0,1\}^{n \times n}$, and notice that $\mB_i$ are i.i.d. random self-adjoint positive
semi-definite matrices. Note that $\norm{\mB_i} \le \norm{\mB_i}_F \le n$, and that $\lambda_{\min}(\sum_{i=1}^k \mean{}{\mB_i}) = k \mu$.
Therefore, by using matrix Chernoff bound, we get that:
\begin{align*}
\prob{}{\lambda_{\min}(\sum_{i=1}^{k} \mB_i) \le (1-\epsilon)k\mu}
\le n \cdot
\left[ \frac{e^{-\epsilon}}{(1-\epsilon)^{(1-\epsilon)}} \right]^{\frac{k}{n}\mu}
\end{align*}
Taking $\epsilon = \frac{1}{2}$ we get that:
\begin{align*}
\prob{}{\lambda_{\min}(\sum_{i=1}^{k} \mB_i) \le \frac{k}{2}\mu}
\le n \cdot
\left(\frac{e}{2} \right)^{-\frac{k}{2n}\mu} \le \delta
\end{align*}
Since we have $\mA^\top \mA = \sum_{i=1}^k \mB_i$ and
$\sigma_{\min}(\mA)^2 = \lambda_{\min}(\mA^\top \mA)$, this completes the proof.

\end{proof}

The next two lemmas show that for this model, none of the $n$ clusters are split
by any of the $k$ filters, with probability $>1-\delta$.

\begin{lemma}
Let $\vu \sim \text{Uni}(\mathbb{S}^{d-1})$, $\vx \in \mathbb{S}^{d-1}$ and $r>0$. Define
$z = \inner{\vu, \vx}$. Then $\Pr(-r\le z \le r) \le r\sqrt{d}$.
\end{lemma}
\begin{proof}
Let $t = \frac{z+1}{2}$. It is well known that $t \sim \text{Beta}(\frac{d-1}{2}, \frac{d-1}{2})$.
We shall start by bounding the Beta function at $B(\frac{d-1}{2}, \frac{d-1}{2})$ with
the following version of Stirling's approximation:
\begin{align*}
\sqrt{\frac{2\pi}{x}}\left(\frac{x}{e}\right)^{x}\le\Gamma\left(x\right)\le\sqrt{\frac{2\pi}{x}}\left(\frac{x}{e}\right)^{x}e^{\frac{1}{12x}}
\end{align*}

\begin{align*}
B\left(\frac{d-1}{2},\frac{d-1}{2}\right)
 &=   \frac{\Gamma\left(\frac{d-1}{2}\right)^{2}}
           {\Gamma\left(d-1\right)} \\
 &\ge \frac{\left(\sqrt{\frac{4\pi}{d-1}}\left(\frac{d-1}{2e}\right)^{\frac{d-1}{2}}\right)^{2}}
           {\sqrt{\frac{2\pi}{d-1}}\left(\frac{d-1}{e}\right)^{d-1}e^{\frac{1}{12\left(d-1\right)}}} \\
 &=   \frac{\frac{4\pi}{d-1}\left(\frac{d-1}{2e}\right)^{d-1}}
           {\sqrt{\frac{2\pi}{d-1}}\left(\frac{d-1}{e}\right)^{d-1}}e^{-\frac{1}{12\left(d-1\right)}} \\
 &=   2\sqrt{\frac{2\pi}{d-1}}\left(\frac{1}{2}\right)^{d-1}e^{-\frac{1}{12\left(d-1\right)}} \\
 &=   \sqrt{\frac{2\pi}{d-1}}\left(\frac{1}{2}\right)^{d-2}e^{-\frac{1}{12\left(d-1\right)}}
\end{align*}
And so,
\begin{align*}
\Pr\left(-r\le z\le r\right)
 &=   \Pr\left(\frac{1-r}{2}\le t\le\frac{1+r}{2}\right) \\
 &\le r \frac{\left(\frac{1}{2}\right)^{\frac{d-3}{2}}\left(\frac{1}{2}\right)^{\frac{d-3}{2}}}
             {B\left(\frac{d-1}{2},\frac{d-1}{2}\right)} \\
 &\le r \frac{\left(\frac{1}{2}\right)^{d-3}}
             {\sqrt{\frac{2\pi}{d-1}}\left(\frac{1}{2}\right)^{d-2}e^{-\frac{1}{12\left(d-1\right)}}} \\
 &=   r \sqrt{d-1}\frac{2}{\sqrt{2\pi}}e^{\frac{1}{12\left(d-1\right)}} \\
 &\le r \sqrt{d}
\end{align*}
Where the last inequality is easily verified numerically.
\end{proof}

\begin{lemma}
Fix $i \in \left[n\right]$, and let $\vu \sim \text{Uni}(\mathbb{S}^{d-1})$. Then, with probability of at least
$1-\frac{1}{kd}\delta$, 
$\forall\left(\vx,y\right)\in S_{i}, \sign\left(\vu^\top\vx\right)=\sign\left(\vu^\top\vv_{i}\right)$.
\end{lemma}
\begin{proof}
By the previous lemma, 
\begin{align*}
\Pr\left(\left|\vu^{\top}\vv_{i}\right| < r \right)
	\le   r\sqrt{d}
	= \frac{\delta}{nk\sqrt{d}}\sqrt{d}
	= \frac{\delta}{kn}
\end{align*}
In the event $\left\{ \left|\vu^{\top}\vv_{i}\right|>r\right\}$ , we get for every $\left(\vx,y\right)\in S_{i}$
\begin{align*}
\left| \vu^{\top}\vx - \vu^{\top} \vv_{i} \right|
  = \left| \vu^{\top} \left( \vx - \vv_{i} \right) \right|
  \underset{{\scriptscriptstyle \text{C.S.}}}{\le}
    \left\Vert \vu\right\Vert \left\Vert \vx-\vv_{i}\right\Vert
  = \left\Vert \vx-\vv_{i}\right\Vert
  \le r
\end{align*}
and so, with probability of at least $1-\frac{1}{kn}\delta$,
$\text{sign}\left(\vu^{\top}\vx\right)=\text{sign}\left(\vu^{\top}\vv_{i}\right)$.
\end{proof}

\begin{lemma}
Let $\vu_{1},\dots,\vu_{k} \sim \text{Uni}(\mathbb{S}^{d-1})$. Then, with probability of at
least $1-\delta$, 
$\forall j \in \left[k\right] \forall i \in \left[n\right] \forall \left(\vx,y\right)\in S_{i}$,
$\sign\left(\vu_{j}^{\top}\vx\right)=\sign\left(\vu_{j}^{\top}\vv_{i}\right)$.
\end{lemma}
\begin{proof}
Union bound on the previous lemma.
\end{proof}

Define 
\begin{align*}
\bar{\mX} = \begin{bmatrix}
 a_{11}\mX_1 & a_{12}\mX_1 & \dots & a_{1k}\mX_1 \\
 a_{21}\mX_2 & a_{22}\mX_2 & \dots & a_{2k}\mX_2 \\
 &  & \vdots\\
 a_{n1}\mX_n & a_{n2}\mX_n & \dots & a_{nk}\mX_n
\end{bmatrix}
\end{align*}
Observe that w.p. $\ge 1-\delta$, according to the last lemma, finding an exact solution to the training problem is equivalent to
finding $\vw_1,\dots,\vw_k \in \reals^d$ such that 
\begin{align*}
\bar{\mX} \begin{bmatrix}
\vw_1 \\
\vw_2 \\
\vdots \\
\vw_k
\end{bmatrix} = \begin{bmatrix}
\mX_1 \ell_1 \\
\mX_2 \ell_2 \\
\vdots\\
\mX_n \ell_n
\end{bmatrix}
\end{align*}

\begin{lemma}
There is at least one solution to the above equation set.
\end{lemma}
\begin{proof}
Because $\rank\left(\mA\right)=n$, there is a matrix $\mB=\left[b_{ij}\right]\in\reals^{n\times k}$ such that
$\mA \mB^\top=\mI_n$. Equivalently, for every $i,i'\in\left[n\right]$, $\sum_{j=1}^{k}b_{i'j}a_{ij}=\1_{i=i'}$.
For every $j\in\left[k\right]$, let $\vw_j=\sum_{i'=1}^{n}b_{i'j}\ell_{i'}$. Now, for every $i\in\left[n\right]$,
\begin{align*}
\sum_{j=1}^{k}a_{ij}\mX_i\vw_{j}
 &= \mX_i \left( \sum_{j=1}^{k}a_{ij}\vw_j \right) \\
 &= \mX_i \left( \sum_{j=1}^{k}a_{ij} \left(\sum_{i'=1}^{n} b_{i'j} \ell_{i'}\right)\right) \\
 &= \mX_i \left( \sum_{i'=1}^{n} \left( \sum_{j=1}^{k}a_{ij}b_{i'j}\right)\ell_{i'}\right) \\
 &= \mX_i \left(\sum_{i'=1}^{n}\1_{i=i'}\ell_{i'}\right) \\
 &= \mX_i\ell_i
\end{align*}
\end{proof}

\begin{lemma}
Every exact solution $\vw_1,\dots,\vw_k$ gives the correct prediction for $\tilde{\vx}$.
\end{lemma}
\begin{proof}
Because $\vw_1,\dots,\vw_k$ is an exact solution, 
\begin{align*}
\sum_{j=1}^{k}a_{qj}\mX_q\vw_j
  = \mX_{q}\left(\sum_{j=1}^{k}a_{qj}\vw_j\right)
  = \mX_q\ell_q 
\end{align*}
Because $\tilde{\vx}^\top\in\text{rowspan}\left(\mX_q\right)$,
\begin{align*}
\sum_{j=1}^{k}a_{qj}\tilde{\vx}^\top\vw_j
  = \tilde{\vx}^\top \left(\sum_{j=1}^{k}a_{qj}\vw_j\right)
  = \tilde{\vx}^\top\ell_{q}=\tilde{y}  
\end{align*}
As required.
\end{proof}

\subsection{Proof of Lemma \ref{lem:mu_min}}
\begin{proof}
Denote $\sigma(x) := \1_{x \ge 0}$.
Let $\epsilon = \frac{1}{2n}$.
Fix some $i \ne j$. Notice that using Hoeffding's inequality, we get that:
\begin{align*}
\prob{}{|\inner{\vv_i, \vv_j}| \ge \epsilon}
\le 2 \exp(-8d \epsilon^2) \le \frac{\delta}{n^2}
\end{align*}
Using the union bound we get that with probability at least $1-\delta$,
for all $i \ne j$, we have $|\inner{\vv_i, \vv_j}| \le \epsilon$.
We assume that this property holds.

Now, we have:
\[
\mH_{i,j} = \mean{\vu \sim N(0,I_d)}{\sigma(\vu^\top \vv_i) \sigma(\vu^\top \vv_j)}
= \frac{1}{2} - \frac{\arccos \inner{\vv_i, \vv_j}}{2 \pi}
\]
Therefore, $\mH_{i,i} = \frac{1}{2}$, and also:
\[
|\mH_{i,j}-\frac{1}{4}| = |\frac{1}{4} - \frac{\arccos \inner{\vv_i, \vv_j}}{2\pi}|
= \frac{1}{2\pi} | \frac{\pi}{2} - \arccos \inner{\vv_i, \vv_j}|
\le \frac{1}{2\pi} \frac{\pi}{2} |\inner{\vv_i, \vv_j}|
\le \frac{1}{4} \epsilon
\]
Where we use $|\arccos(x) - \frac{\pi}{2}| \le \frac{\pi}{2} |x|$.
Denote $\mT = \frac{1}{4}I + \frac{1}{4} \1 \1^\top$,
and we therefore have:
\[
\norm{\mH - \mT} \le \norm{\mH- \mT}_{F} \le
\frac{n\epsilon}{4} \le \frac{1}{8}
\]
Notice that $T$ is invertible, and $\mT^{-1} = 4I - \frac{4}{d+1} \1 \1^\top$
(this is easy to check). By simple calculation we get that
$\norm{\mT^{-1}} = 4$ (see below).
Therefore, we get that $\norm{\mH-\mT} \le \norm{\mT^{-1}}^{-1}$, so $\mH$
is invertible, and we have:
\[
\norm{\mH^{-1}} = \norm{\sum_{j=0}^{\infty} (\mT^{-1}(\mT-\mH))^j \mT^{-1} }
\le \norm{\mT^{-1}} \sum_{j=0}^\infty (\norm{\mT^{-1}}\norm{\mT-\mH})^j 
\le \norm{\mT^{-1}} \sum_{j=0}^\infty (\frac{1}{2})^j 
\le 8
\]
Therefore $\mu_{\min} = \lambda_{\min}(\mH) \ge \frac{1}{8}$.
\end{proof}

\begin{lemma}
$\norm{4\mI - \frac{4}{d+1} \1 \1^\top} = 4$.
\end{lemma}
\begin{proof}
Let $\vx$ be a unit vector. Denote $\vx = \vx_{\1\1^\top} + \vx'$ such that
$\vx_{\1\1^\top} \in \text{span}\left\{\1\right\}$ and $\inner{\vx', \1} = 0$. Now,
\begin{align*}
\norm{(\mI - \frac{1}{d+1} \1 \1^\top)\vx}
 &= \norm{(\mI - \frac{1}{d+1} \1 \1^\top)\vx_{\1\1^\top} + \vx'} \\
 &= \norm{\vx' - \frac{d}{d+1} \vx_{\1\1^\top}} \\
 &\le \norm{\vx'} + \frac{d}{d+1} \norm{\vx_{\1\1^\top}} \\
 &\le \norm{\vx'} + \norm{\vx_{\1\1^\top}} \\
 &= \norm{\vx} = 1
\end{align*}
With equality iff $\norm{\vx_{\1\1^\top}} = 0$.
\end{proof}

\section{Proof of Theorem \ref{thm:shallow_generalization}}

\begin{proof} of Theorem \ref{thm:shallow_generalization}.
Let $S = \{(\vx_1, y_1), \dots, (\vx_m, y_m)\}$ where $S \sim \mathcal{D}^m$.
We denote $\mH^{\infty}$ the Gram matrix such that
$\mH^{\infty}_{i,j} = \kappa(\vx_i, \vx_j)$.
Observe that $\lambda_{\min}(\mH^{\infty}) = \lambda(\mX) > 0$,
so $\mH^{\infty}$ is full-rank.
Define $\varphi : \mathcal{X} \to \mathcal{H}_\kappa$ such that
$\varphi(\vx) = \kappa(\cdot, \vx)$.
Observe the minimization problem:

\begin{align*}
\hat{f} = \argmin \sum_{i=1}^m \frac{1}{2}(f(\vx_i)-y_i)^2 \\
\textnormal{s.t.}~ f(\vx) = \sum_{j=1}^m \hat{w}_j \varphi(\vx_i)
\end{align*}

The solution to this minimization problem is given by:
\[
\hat{\vw} = (\mH^{\infty})^{-1} \vy
\]
Now, calculating the norm of $\hat{f}$ we get:
\begin{align*}
\norm{\hat{f}}_{\kappa}^2 &=
\inner{\sum_{i=1}^m \hat{w}_i \varphi(\vx_i), \sum_{j=1}^m \hat{w}_j \varphi(\vx_j)} \\
&= \sum_{i,j=1}^m \hat{w}_i \hat{w}_j \inner{\varphi(\vx_i), \varphi(\vx_j)} \\
&= (\hat{\vw})^\top \mH^{\infty} \hat{\vw} = \vy^\top (\mH^\infty)^{-1} \vy
\end{align*}
Observe that $\hat{f}$ is the projection of $f^{*}$ onto the space spanned
by $\{\varphi(\vx_1), \dots, \varphi(\vx_m)\}$ (since the loss of this projection on the space must be zero, and the only choice for such function is $\hat{f}$). Therefore:
\[
\sqrt{\vy^\top (\mH^{\infty})^{-1} \vy} = \norm{\hat{f}}_\kappa \le 
\norm{f^{*}}_\kappa \le M
\]

Now, observe the GaLU optimization problem (where $\bar{\mX}, \mH$ are as defined previously):
\[
\vw^* = \argmin \sum_{i=1}^m \frac{1}{2} (\bar{X}_i \vw - y_i)^2
\]
The solution is given by:
\[
\vw^* = \bar{\mX} (\bar{\mX}\bar{\mX}^\top)^{-1} \vy
\]
So we have:
\[
\norm{\vw^*}^2 = \vy^\top (\bar{\mX}\bar{\mX}^\top)^{-1} \vy
= \vy^\top \mH^{-1} \vy
\]

To finish the argument, we need to relate $\mH^{-1}$ to
$(\mH^{\infty})^{-1}$. To do this, we start by bounding $\norm{\mH - \mH^{\infty}}$.
Recall that we define $\mH = \frac{1}{k} \sum_{i=1}^k \mH^{(i)}$,
and that $\mH^{\infty} = \mean{}{\mH} = \mean{}{\mH^{(i)}}$.
We also have $\norm{\mH^{(i)}} \le \norm{\mX}^2 := R$ and therefore
$\norm{\mH^{\infty}} \le \norm{\mH^{(i)}} \le R$.
Now, denote $\mY^{(i)} = \frac{1}{k}\mH^{(i)} - \frac{1}{k}\mH^{\infty}$ so we have
$\norm{\mY^{(i)}} \le \frac{2}{k} R$. Also, we have $\mean{}{\mY^{(i)}} = 0$,
and $\mY^{(i)}$ are i.i.d random self-adjoint matrices, so we can use
Matrix Hoeffding inequality and get for every $r > 1$:
\begin{align*}
\prob{}{\norm{\mH-\mH^{\infty}} \ge \frac{1}{r} \lambda(\mX)} &= 
\prob{}{\norm{\sum_{i=1}^k \mY^{(i)}} \ge \frac{1}{r} \lambda(\mX)} \\
&\le m \cdot \exp(-\frac{k\lambda(\mX)^2}{32r^2R^2})
\end{align*}
Therefore, if we take $k \ge \frac{32r^2\norm{\mX}^4}{\lambda(\mX)^2} \log(m/\delta)$
we get that the above happens w.p at most $1-\delta$. So from now we assume that
$\norm{\mH-\mH^{\infty}} \le \frac{1}{r} \lambda(X)$.

Now, recall the following property:
for two square matrices $\mA,\mB$ such that $\mA$ is invertible,
if $\norm{\mB-\mA} \le \norm{\mA^{-1}}^{-1}$ then $\mB$ is invertible 
and $\mB^{-1} = \mA^{-1} \sum_{n=0}^\infty (\mB-\mA)\mA^{-1}$.
In our case, we know (assume) that $\mH^{\infty}$ is invertible, and
we showed that w.h.p:
\[
\norm{\mH-\mH^{\infty}} \le \frac{1}{r} \lambda(\mX) = \frac{1}{r}\lambda_{\min}(\mH^{\infty})
= \frac{1}{r}\norm{(\mH^{\infty})^{-1}}^{-1}
\]
therefore we get:
\begin{align*}
\norm{\mH^{-1}-(\mH^{\infty})^{-1}} &= \norm{(\mH^{\infty})^{-1}
\left( \sum_{n=0}^\infty
((\mH-\mH^{\infty})(\mH^{\infty})^{-1})^n -I\right)} \\
&\le \norm{(\mH^{\infty})^{-1}}\sum_{n=1}^\infty (\norm{\mH-\mH^{\infty}}\norm{(\mH^{\infty})^{-1}})^n \\
&\le \norm{(\mH^{\infty})^{-1}} \sum_{n=1}^\infty r^{-n} \\
&= \frac{1}{r-1}
\norm{(\mH^{\infty})^{-1}} = \frac{1}{(r-1)\lambda(\mX)}
\end{align*}

Combining this with what we have shown previously we get:
\begin{align*}
\norm{\vw^*}^2 &= \vy^\top \mH^{-1} \vy \\
&\le \vy^\top (\mH^{\infty})^{-1} \vy
+ \norm{\vy} \norm{(\mH^{\infty})^{-1} - \mH^{-1}} \norm{\vy} \\
&\le M^2 + \frac{m}{(r-1)\lambda(\mX)} 
\end{align*}

Now if we choose $r \ge \frac{m}{M^2 \lambda(\mX)} + 1$ we get that
$\norm{\vw^*}^2 \le 2M^2$.
Denote:
\[
\mathcal{H}_{\sqrt{2}M} = \{\frac{1}{\sqrt{k}} \sum_{j=1}^k g_{\vw_j, \vu_j}
~|~ \norm{\vw} = \sqrt{\sum_{j=1}^k \norm{\vw_j}^2} \le \sqrt{2}M \}
\]
This is the hypothesis class of (normalized) GaLU networks with norm bounded by
$\sqrt{2}M$.
The Radamacher complexity of $\mathcal{H}_{\sqrt{2}M}$ is given by:
\begin{align*}
\mathcal{R}_m(\mathcal{H}_{\sqrt{2}M}) &=
\sup_{\vx_1, \dots, \vx_m\ \in \mathcal{X}}
\mean{\sigma \sim U(\{\pm1\}^m)}{\sup_{h\in \mathcal{H}_{\sqrt{2}M}}
\frac{1}{m}\sum_{i=1}^m h(\vx_i)\sigma_i} \\
&= \sup_{\vx_1, \dots, \vx_m\ \in \mathcal{X}}
\mean{\sigma \sim U(\{\pm1\}^m)}{\sup_{\norm{\vw} \le \sqrt{2}M}
\frac{1}{m}\sum_{i=1}^m \frac{1}{\sqrt{k}} \sigma_i \Phi_\vu(\vx_i)^\top \vw} \\
\end{align*}
Notice that we have $\norm{\frac{1}{\sqrt{k}}\Phi_\vu(\vx_i)} = 
\sqrt{\frac{1}{k} \sum_{j=1}^k \norm{\1_{\vu_j^\top \vx_i \ge 0} \vx_)}^2}
\le \norm{\vx_i} \le 1$. Therefore, from standard Rademacher analysis
for linear functions with bounded norm (for example in \cite{MLbook}), we get
that $\mathcal{R}_m(\mathcal{H}_{\sqrt{2}M}) \le \frac{\sqrt{2}M}{\sqrt{m}}$.
Notice that the square loss function $\ell(y,\hat{y}) = \frac{1}{2}(y-\hat{y})^2$
is $1$-smooth. Since for every $h \in \mathcal{H}_{\sqrt{2}M}$ and $\vx \in \mathcal{X}$
we have $|h(\vx)| \le \sqrt{2}M$, and we assume that $\mathcal{Y} \subseteq [-1,1]$,
we can assume that the loss function $\ell$ is defined over
$[-\sqrt{2}M, \sqrt{2}M] \times [-1,1]$.
Then for $\hat{y},\hat{y}' \in [-\sqrt{2}M,\sqrt{2}M], y \in [-1,1]$ we have:
\[
|\ell(\hat{y},y)-\ell(\hat{y}',y)| = \frac{1}{2}|\hat{y}^2 -\hat{y}y -\hat{y}'^2
-\hat{y}'y| \le 2M^2 + \sqrt{2}M
\]
Now, for $\hat{h} \in \mathcal{H}_{\sqrt{2}M}$, the GaLU network with weights $\vw^*$,
using Theorem 1 in \cite{srebro_optimistic_2010},
we get with probability at least $1-\delta$ a generalization bound of:
\[
L_{\mathcal{D}}(\hat{h}) \le C \left( \frac{2M^2 \log^3m + (2M^2 + \sqrt{2}M)
\log(1/\delta)}{m} \right)
\]
For some constant $C$.
\end{proof}

\section{Proofs of Section \ref{sec:relu_relation}}
\subsection{Proof of Theorem \ref{thm:failure_equivalence}}
\begin{proof}
Denote $\sigma(x) = \1_{x \ge 0}$ the gate of the GaLU network and
$\phi(x) = [x]_+ = \1_{x \ge 0} \cdot x$ the ReLU activation.
By our assumption, the output of the network is bounded in $[-1,1]$ upon initialization,
so:
\[
L_S(\mathcal{N}^G_{\mW,\mU, \alpha}) = -\frac{1}{m} \sum_{i=1}^m y_i \mathcal{N}^R_{\mU, \alpha}(\vx_i)
\]
Therefore we get for every $j$:
\begin{align*}
\frac{\partial}{\partial \vw_j} L_S(\mathcal{N}^G_{\mW, \mU, \alpha})
&= -\frac{1}{m\sqrt{k}} \sum_{i=1}^m y_i \alpha_j\sigma(\vx_i^\top \vu_j) \vx_i \\
&= -\frac{1}{m\sqrt{k}} \sum_{i=1}^m y_i \alpha_j\phi'(\vx_i^\top \vu_j) \vx_i \\
&= \frac{\partial}{\partial \vu_j}\mathcal{L}_S(\mathcal{N}^G_{\mU, \alpha})
\end{align*}
And the result immediately follows.
\end{proof}

\subsection{Proof of Theorem \ref{thm:relu_similarity}}
Recall that for some vectors $\vu_1, \dots, \vu_k \in \reals^d$,
we denote $\Phi: \mathcal{X} \to \reals^{dk}$ where
 $\Phi_{\mU}(\vx) = \frac{1}{\sqrt{k}}
[\1_{\vu_1^\top \vx \ge 0} \vx, \dots, \1_{\vu_j^\top \vx \ge 0} \vx]$.
For some $\mW = [\vw_1, \dots, \vw_k]$ where $\vw_i \in \reals^d$,
and $\alpha = [\alpha_1, \dots, \alpha_k]$ where $\alpha_i \in \reals$,
we define a vector
$v(\mW, \alpha) = [\alpha_1\vw_1 ~ \dots ~ \alpha_k \vw_k]\in \reals^{dk}$.
Now, we can write:
\[
\mathcal{N}_{\mW,\mU,\alpha}^G(\vx) =
\frac{1}{\sqrt{k}} \Phi_\mU(\vx)^{\top} v(\mW, \alpha), ~
\mathcal{N}_{\mU,\alpha}^R(\vx) =
\frac{1}{\sqrt{k}} \Phi_\mU(\vx)^{\top} v(\mU, \alpha)
\]
We start with the following lemma:
\begin{lemma}
\label{lem:gates_perturbation}
Fix $\delta > 0$, let $k \ge \frac{\pi}{\sqrt{6}d \epsilon^2} \left( \log(2/\delta)
+ d\log(3/\epsilon) \right)$,
and we assume $d > \log(2k/\delta)$.
Assume we draw $\vu_1, \dots, \vu_k \sim \mathcal{N}(0,I_d)$.
Let $\vw_1, \dots, \vw_k$ be some vectors such that for all $j \in [k]$
we have $\norm{\vu_j - \vw_j} \le \epsilon$, for some $\epsilon > 0$.
Then with probability at least $1-\delta$,
we have $\norm{\Phi_{\mU} - \Phi_{\mW}}_\infty \le \sqrt{\frac{5\sqrt{3d}\epsilon}{\sqrt{2\pi}}}$.
\end{lemma}
\begin{proof}
Fix $B = \sqrt{\sqrt{6}d}$, and from Lemma B.12 in \cite{MLbook},
we have that:
\[
\prob{}{\norm{\vu_j}^2 \ge \sqrt{6}d} \le e^{-d} \le \frac{\delta}{2k}
\]
Using the union bound, we have with probability at least $1-\frac{\delta}{2}$,
for all $j \in [k]$ we have $\norm{\vu_j} \le B$, so we assume this holds.
Let $\delta'=\frac{1}{2}(3/\epsilon)^{-d} \delta$.
Fix some $\vx \in \mathcal{X} = \mathbb{S}^{d-1}$, and fix some $j \in [k]$.
Notice that $\vu_j^\top \vx \sim \mathcal{N}(0,1)$, and therefore:
\[
\prob{\vu_j \sim \mathcal{N}}{|\vu_j^\top \vx| \le 2B\epsilon} \le \frac{4B\epsilon}{\sqrt{2\pi}}
\]
Denote $S_\vx = \frac{1}{k} \sum_{j=1}^k \1_{|u_j^\top \vx| \le 2B\epsilon}$,
so $\mean{}{S_\vx} \le \frac{4B\epsilon}{\sqrt{2\pi}}$,
and from Hoeffding's inequality we have:
\[
\prob{}{S_\vx \ge \frac{5B\epsilon}{\sqrt{2\pi}}}
=\prob{}{S_\vx \le \mean{}{S_\vx} + \frac{B\epsilon}{\sqrt{2\pi}}}
\le \exp\left(-2k \frac{B^2\epsilon^2}{2\pi} \right) \le \delta'
\]
For every $\vx' \in \mathcal{X}$
with $\norm{\vx-\vx'} \le \epsilon$, if $|\vu_j^\top \vx| > 2B\epsilon$ then we have:
\[
|\vu_j^\top\vx'| \ge |\vu_j^\top\vx| - \norm{\vx-\vx'}\norm{\vu_j}
\ge |\vu_j^\top\vx| - B\epsilon > B\epsilon \ge \epsilon
\]
For such $\vx'$ we have
$|\vu_j^\top\vx'-\vw_j^\top\vx'| \le \norm{\vu_j-\vw_j}\norm{\vx'} \le \epsilon$,
so $\sign(\vu_j^\top\vx') \ne \sign(\vw_j^\top \vx')$ only if
$|\vu_j^\top\vx'| \le \epsilon$.

Therefore, w.p at least $1-\delta'$ we have for every $\vx' \in \mathcal{X}$
with $\norm{\vx-\vx'} \le \epsilon$:
\begin{align*}
\norm{\Phi_\mU(\vx')-\Phi_\mW(\vx')}^2 &=
\frac{1}{k} \sum_{j=1}^k \1_{\sign(\vu_j^\top\vx') \ne \sign(\vw_j^\top \vx')}
\norm{\vx}^2 \\
&\le \frac{1}{k} \sum_{j=1}^k \1_{|u_j^\top \vx'| \le \epsilon} \\
&\le \frac{1}{k} \sum_{j=1}^k \1_{|u_j^\top \vx| \le 2B\epsilon} \\
&= S_\vx \le \frac{5B\epsilon}{\sqrt{2\pi}}
\end{align*}
Now, there is an $\epsilon$-net of $\mathcal{X}$ of size at most $(3/\epsilon)^d$,
and we denote this net by $N \subseteq \mathcal{X}$. From the union bound
we get that with probability at least $1-(3/\epsilon)^d\delta' = 1-\delta/2$
we have for all $\vx \in N$, and for every $\vx' \in \mathcal{X}$
with $\norm{\vx-\vx'} \le \epsilon$, that:
\[
\norm{\Phi_\mU(\vx')-\Phi_\mW(\vx')} \le \sqrt{\frac{5B\epsilon}{\sqrt{2\pi}}}
\]
In this case, the above inequality holds for every $\vx' \in \mathcal{X}$,
and we get the required.
\end{proof}

The above shows that small perturbation in $\vu_j$-s implies small perturbation
of the map $\Phi_\vu$.
Now, fix some $L$-Lipschitz loss $\ell : \reals \times \mathcal{Y} \to \reals$,
and denote $L_{\mathcal{D}}(f) = 
\mean{(x,y) \sim \mathcal{D}}{\ell(f(x),y)}$.
Fix some $\vv \in \reals^{dk}$ and observe the two functions
$h_{\vu,\vv}(\vx) = \Phi_{\vu}(\vx)^\top \vv$ and
$h_{\vw,\vv}(\vx) = \Phi_{\vw}(\vx)^\top \vv$.
Then we have:
\begin{align*}
|L_{\mathcal{D}}(h_{\vu,\vv})- L_{\mathcal{D}}(h_{\vw,\vv})|
&= |\mean{(x,y) \sim \mathcal{D}}{\ell(h_{\vu,\vv}(x),y)}
- \mean{(x,y) \sim \mathcal{D}}{\ell(h_{\vw,\vv}(x),y)}| \\
&= |\mean{(x,y) \sim \mathcal{D}}{\ell(h_{\vu,\vv}(x),y) - \ell(h_{\vw,\vv}(x),y)}| \\
&\le \mean{(x,y) \sim \mathcal{D}}{|\ell(h_{\vu,\vv}(x),y) - \ell(h_{\vw,\vv}(x),y)|} \\
&\le \mean{(x,y) \sim \mathcal{D}}{L|h_{\vu,\vv}(x) - h_{\vw,\vv}(x)|} \\
&= \mean{(x,y) \sim \mathcal{D}}{L|\Phi_\vu(\vx)^\top \vv - \Phi_\vw(\vx)^\top \vv|} \\
&\le L \norm{\vv} \norm{\Phi_\vu - \Phi_\vw}_\infty  \\
\end{align*}

And this gives the following:
\begin{align*}
L_{\mathcal{D}}(\mathcal{N}_{\mW^*,\mU,\alpha^*}^G) &\le 
L_{\mathcal{D}}(\mathcal{N}_{\mU^*,\mU,\alpha^{**}}^G) \\
&\le L_{\mathcal{D}}(\mathcal{N}_{\mU^*,\mU^*,\alpha^{**}}^G) 
+ L \norm{\frac{1}{\sqrt{k}} v(\mU^*,\alpha^{**})}
\norm{\Phi_{\vu^*} - \Phi_{\vu}}_\infty \\
&= L_{\mathcal{D}}(\mathcal{N}_{\mU^*,\mU^*,\alpha^{**}}^G) 
+ L \sqrt{\frac{1}{k} \sum_{i=1}^k \norm{\alpha_i^{**} \vu_i^*}^2}
\norm{\Phi_{\vu^*} - \Phi_{\vu}}_\infty \\
&\le L_{\mathcal{D}}(\mathcal{N}_{\mU^*,\mU^*,\alpha^{**}}^G) 
+ L \max_i \norm{\alpha_i^{**} \vu_i^*}
\norm{\Phi_{\vu^*} - \Phi_{\vu}}_\infty
\end{align*}
Now, observing that $g_{\vu^*, \vu^*} = f_{\vu^*}$, and using Lemma \ref{lem:gates_perturbation} completes the proof.

\clearpage

\section{Experiments}
\label{sec:experiments}
We showed theoretical results that establish the relation between ReLU
and GaLU. To complete the picture, we now turn to evaluate this relation empirically.
We start with a memorization experiment, where the task at hand is to memorize
a randomly generated sample (as described in \ref{sec:shallow_optimization}).
In this experiment, we draw $m$ examples in dimension $d$,
where both the input and the label are sampled from a Gaussian distribution.
Recall that for this case, we theoretically
showed that a GaLU network needs $\tilde{\Omega}(\frac{m}{d})$ neurons to reach
zero loss. We train both GaLU and ReLU network on this task,
with Adam optimizer, batch size 128 and learning rate of $0.001$
for $100k$ iterations.
Using binary search, we find the minimal $k$ to reach MSE loss $<0.01$.
Each experiment is repeated $5$ times.
We see that for both the ReLU and GaLU networks we get $k \simeq \frac{m}{d}$,
for different sample sizes. The results of this experiments are shown in
\figref{fig:memorization}.

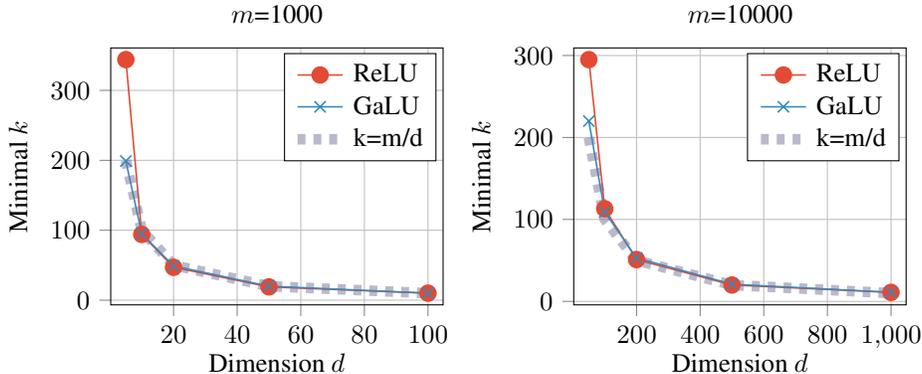
\begin{figure}
\begin{center}
% This file was created by matplotlib2tikz v0.6.18.
\begin{tikzpicture}

\definecolor{color1}{rgb}{0.203921568627451,0.541176470588235,0.741176470588235}
\definecolor{color0}{rgb}{0.886274509803922,0.290196078431373,0.2}
\definecolor{color2}{rgb}{0.596078431372549,0.556862745098039,0.835294117647059}

\begin{axis}[
width=6cm,
height=5cm,
legend cell align={left},
legend entries={{ReLU},{GaLU},{k=m/d}},
tick align=outside,
tick pos=left,
title={$m$=1000},
xlabel={Dimension $d$},
xmajorgrids,
xmin=0.25, xmax=104.75,
ylabel={Minimal $k$},
ymajorgrids,
ymin=-6.7, ymax=360.7
]
\addplot [semithick, color0, mark=*, mark size=3, mark options={solid}]
table [row sep=\\]{%
5	344 \\
10	94 \\
20	47 \\
50	19 \\
100	10 \\
};
\addplot [semithick, color1, mark=x, mark size=3, mark options={solid}]
table [row sep=\\]{%
5	199 \\
10	95 \\
20	49 \\
50	20 \\
100	10 \\
};
\addplot [line width=4pt, CadetBlue, opacity=0.5, dashed]
table [row sep=\\]{%
5	200 \\
10	100 \\
20	50 \\
50	20 \\
100	10 \\
};
\end{axis}

\end{tikzpicture}
% This file was created by matplotlib2tikz v0.6.18.
\begin{tikzpicture}

\definecolor{color1}{rgb}{0.203921568627451,0.541176470588235,0.741176470588235}
\definecolor{color0}{rgb}{0.886274509803922,0.290196078431373,0.2}
\definecolor{color2}{rgb}{0.596078431372549,0.556862745098039,0.835294117647059}

\begin{axis}[
width=6cm,
height=5cm,
legend cell align={left},
legend entries={{ReLU},{GaLU},{k=m/d}},
tick align=outside,
tick pos=left,
title={$m$=10000},
xlabel={Dimension $d$},
xmajorgrids,
xmin=2.5, xmax=1047.5,
ylabel={Minimal $k$},
ymajorgrids,
ymin=-4.25, ymax=309.25
]
\addplot [semithick, color0, mark=*, mark size=3, mark options={solid}]
table [row sep=\\]{%
50	295 \\
100	113 \\
200	51 \\
500	20 \\
1000	11 \\
};
\addplot [semithick, color1, mark=x, mark size=3, mark options={solid}]
table [row sep=\\]{%
50	220 \\
100	109 \\
200	53 \\
500	21 \\
1000	11 \\
};
\addplot [line width=4pt, CadetBlue, opacity=0.5, dashed]
table [row sep=\\]{%
50	200 \\
100	100 \\
200	50 \\
500	20 \\
1000	10 \\
};
\end{axis}

\end{tikzpicture}
\end{center}
\caption{Minimal number of neurons to reach MSE $< 0.01$,
for different sample sizes $m$.}
\label{fig:memorization}
\end{figure}

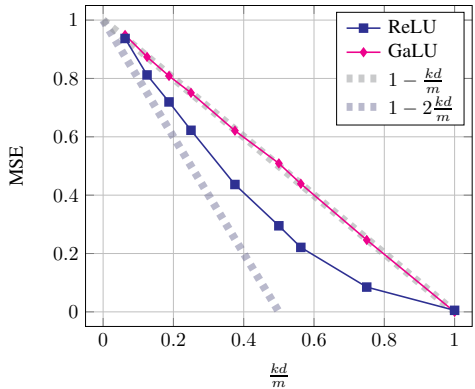
\begin{wrapfigure}{r}{0.5\textwidth}
  \begin{center}
    \begin{tikzpicture}[scale=0.75]
      \begin{axis}[
        xlabel={$\frac{kd}{m}$},
        xmajorgrids,
        xmin=-0.05, xmax=1.05,
        ylabel={MSE},
        ymajorgrids,
        ymin=-0.05, ymax=1.05,
        legend cell align={left},
        reverse legend
      ]
        \addplot [line width=4pt, CadetBlue, opacity=0.5, dashed] table [row sep=\\]{%
          0	  1 \\
          0.5	0 \\
        };
        \addplot [line width=4pt, Gray, opacity=0.5, dashed] table [row sep=\\]{%
          0	1 \\
          1	0 \\
        };
        \addplot [thick, Magenta, mark=diamond*] table [row sep=\\]{%
          0.0625	0.949 \\
          0.125 	0.874 \\
          0.1875	0.8085 \\
          0.25	  0.750666666666667 \\
          0.375	  0.621 \\
          0.5	    0.509 \\
          0.5625	0.439 \\
          0.75	  0.2465 \\
          1	      4.14e-05 \\
        };
        \addplot [thick, Blue, mark=square*] table [row sep=\\]{%
          0.0625	0.937 \\
          0.125	  0.812 \\
          0.1875	0.7195 \\
          0.25	  0.622333333333333 \\
          0.375	  0.4365 \\
          0.5	    0.295 \\
          0.5625	0.221 \\
          0.75	  0.08495 \\
          1	      0.00527 \\
        };
        \legend{$1 - 2\frac{kd}{m}$, $1-\frac{kd}{m}$, GaLU, ReLU};
      \end{axis}
    \end{tikzpicture}
  \end{center}
  \caption{Comparison of GaLU and ReLU networks with a single hidden layer 
           in the under-parametrized case.}
  \label{fig:r1_regression}
\end{wrapfigure}

Next, we turn to observing a memorization task in the under-parametrized case.
In this experiment we observe the loss of the network different choices of $k$
where $k \le \frac{m}{d}$. In this case, the loss of the GaLU network
behaves like $1-\frac{kd}{m}$, as predicted by our theoretical analysis.
A ReLU network, on the other hand, achieves slightly better performance than the GaLU
network in this regime, but its loss is lower bounded by
$1-2\frac{kd}{m}$. In other words, a GaLU network with $2k$ neurons
achieves the same performance as a ReLU network with $k$ neurons,
so a ReLU network gives only a constant gain in parameter utilization.
The results of this experiments are shown in \figref{fig:r1_regression}.

Going beyond a pure memorization task, we observe the behavior of ReLU
and GaLU on linearly separable data. It has been shown \cite{brutzkus2017sgd}
that linearly separable data is learnable by neural-networks,
with sample complexity similar to a linear classifier.
Therefore, this task is an interesting benchmark to compare the performance
of ReLU and GaLU networks. In this experiment we draw examples from a
Gaussian distributions in $\reals^{100}$ and uniformly choose a vector
$w$ on the sphere, to be the linear separator. 
We use $50k$ examples for train and $10k$ examples for test,
filtering only examples with margin $\ge 0.01$.
Here we train both GaLU and ReLU networks with the Adam optimizer,
using learning rate $0.001$, for $100k$ iterations and batch size $128$,
comparing different network widths.
Each experiment is repeated $3$ times, and the results are averaged over the experiments.
Figure \ref{fig:parity_linear} shows the accuracy on the test set in this
experiment. Note that both ReLU and GaLU achieve very high accuracy,
with visible advantage to the ReLU network.

Next, we observe the performance of GaLU and ReLU on MNIST and
Fashion-MNIST datasets. Training is performed as described previously
A comparison of the performance of various network widths
on the test data is shown in Figure \ref{fig:mnist}.
Again, we observe similar behavior, with ReLU networks
performing slightly better than GaLU.

\begin{figure}
\begin{center}
% This file was created by matplotlib2tikz v0.6.18.
\begin{tikzpicture}

\definecolor{color1}{rgb}{0.203921568627451,0.541176470588235,0.741176470588235}
\definecolor{color0}{rgb}{0.886274509803922,0.290196078431373,0.2}

\begin{groupplot}[group style={group size=2 by 1, horizontal sep=2cm, vertical sep=2cm}]
\nextgroupplot[
width=6cm,
height=5cm,
tick align=outside,
tick pos=left,
title={MNIST},
xlabel={$k$},
xmajorgrids,
xmin=10.4, xmax=133.6,
ylabel={1 - Test Accuracy},
ymajorgrids,
ymin=0, ymax=0.2,
yticklabels={0,0,0.05,0.1,0.15,0.2}
]
\addplot [semithick, color0, mark=*, mark size=3, mark options={solid}, forget plot]
table [row sep=\\]{%
16	0.0656 \\
32	0.0368 \\
64	0.028 \\
128	0.0239 \\
};
\addplot [semithick, color1, mark=x, mark size=3, mark options={solid}, forget plot]
table [row sep=\\]{%
16	0.117 \\
32	0.101 \\
64	0.0612 \\
128	0.0465 \\
};
\nextgroupplot[
width=6cm,
height=5cm,
legend cell align={left},
legend entries={{ReLU},{GaLU}},
legend pos={south east},
tick align=outside,
tick pos=left,
title={Fashion MNIST},
xlabel={$k$},
xmajorgrids,
xmin=10.4, xmax=133.6,
ylabel={1 - Test Accuracy},
ymajorgrids,
ymin=0, ymax=0.2,
yticklabels={0,0,0.05,0.1,0.15,0.2}
]
\addplot [semithick, color0, mark=*, mark size=3, mark options={solid}]
table [row sep=\\]{%
16	0.155 \\
32	0.147 \\
64	0.132 \\
128	0.117 \\
};
\addplot [semithick, color1, mark=x, mark size=3, mark options={solid}]
table [row sep=\\]{%
16	0.18 \\
32	0.177 \\
64	0.164 \\
128	0.146 \\
};
\end{groupplot}

\end{tikzpicture}
\end{center}
\caption{Performance on MNIST and Fashion-MNIST.}
\label{fig:mnist}
\end{figure}
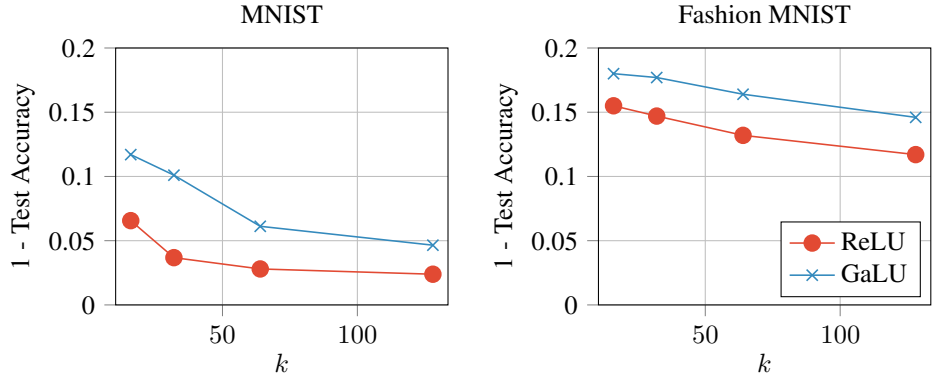

Finally, we move to observing a failure case.
We test GaLU and ReLU networks on the parity task,
which is known to be a hard task for neural-networks in general
\cite{shalev2017failures}.
In this task, we draw uniformly examples s.t $x \sim Uni(\{\pm1\}^{100})$,
and setting the labels to be $y = \prod_{i=1}^{100} x_i$.
So the label of the example $x$ is $1$ if the number of $-1$-s in the 
example is even. Using again $50k$ examples for a training set and
$10k$ examples as a test set, with a training scheme similar to before,
we observe that both GaLU and ReLU networks completely fail in this
task, achieving only chance-level performance.
This is shown in Figure \ref{fig:parity_linear}.

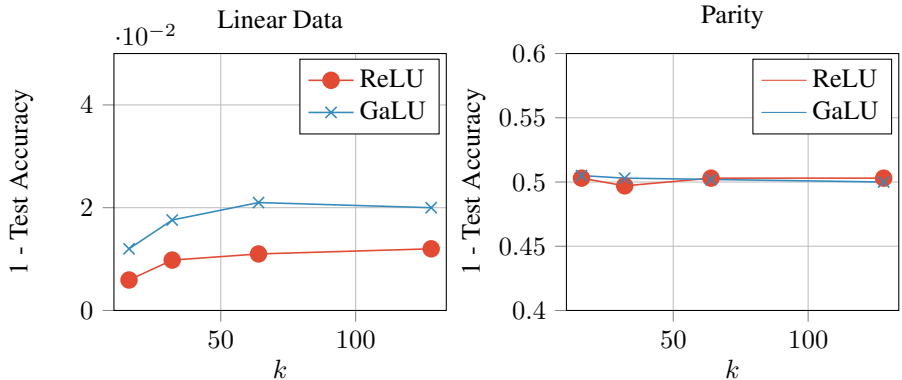
\begin{figure}[h]
\begin{center}
% This file was created by matplotlib2tikz v0.6.18.
\begin{tikzpicture}

\definecolor{color1}{rgb}{0.203921568627451,0.541176470588235,0.741176470588235}
\definecolor{color0}{rgb}{0.886274509803922,0.290196078431373,0.2}

\begin{axis}[
title={Linear Data},
width=6cm,
height=5cm,
legend cell align={left},
legend entries={{ReLU},{GaLU}},
tick align=outside,
tick pos=left,
xlabel={$k$},
xmajorgrids,
xmin=10.4, xmax=133.6,
ylabel={1 - Test Accuracy},
ymajorgrids,
ymin=0, ymax=0.05
]
\addplot [semithick, color0, mark=*, mark size=3, mark options={solid}]
table [row sep=\\]{%
16	0.00593 \\
32	0.0098 \\
64	0.011 \\
128	0.012 \\
};
\addplot [semithick, color1, mark=x, mark size=3, mark options={solid}]
table [row sep=\\]{%
16	0.012 \\
32	0.0176 \\
64	0.021 \\
128	0.02 \\
};
\end{axis}

\end{tikzpicture}
% This file was created by matplotlib2tikz v0.6.18.
\begin{tikzpicture}

\definecolor{color1}{rgb}{0.203921568627451,0.541176470588235,0.741176470588235}
\definecolor{color0}{rgb}{0.886274509803922,0.290196078431373,0.2}

\begin{axis}[
title={Parity},
width=6cm,
height=5cm,
legend cell align={left},
legend entries={{ReLU},{GaLU}},
tick align=outside,
tick pos=left,
xlabel={$k$},
xmajorgrids,
xmin=10.4, xmax=133.6,
ylabel={1 - Test Accuracy},
ymajorgrids,
ymin=0.4, ymax=0.6
]
\addlegendimage{no markers, color0}
\addlegendimage{no markers, color1}
\addplot [semithick, color0, mark=*, mark size=3, mark options={solid}]
table [row sep=\\]{%
16	0.503 \\
32	0.497 \\
64	0.503 \\
128	0.503 \\
};
\addplot [semithick, color1, mark=x, mark size=3, mark options={solid}]
table [row sep=\\]{%
16	0.505 \\
32	0.503 \\
64	0.502 \\
128	0.5 \\
};
\end{axis}

\end{tikzpicture}
\end{center}
\caption{Performance on linearly separable data and on the
parity task.}
\label{fig:parity_linear}
\end{figure}

\end{document}